\documentclass[journal]{IEEEtai}

\usepackage[colorlinks,urlcolor=blue,linkcolor=blue,citecolor=blue]{hyperref}

\usepackage{color,array}

\usepackage{graphicx}

\usepackage{url}            % simple URL typesetting
\usepackage{booktabs}       % professional-quality tables
\usepackage{amsfonts}       % blackboard math symbols
\usepackage{nicefrac}       % compact symbols for 1/2, etc.
\usepackage{microtype}      % microtypography
\usepackage{xcolor}         % colors
\usepackage{float}          % table position
\usepackage{adjustbox}
\usepackage{subcaption}
\usepackage{colortbl}
\usepackage{threeparttable}
\usepackage{bbding}
\usepackage{verbatim}
\usepackage{bm}
\usepackage{multirow}
\usepackage{tabularx}
\usepackage{pifont}
\usepackage{amsmath}
\usepackage{algpseudocode}
\usepackage{amssymb}
\usepackage{hyperref}

\newtheorem{theorem}{Theorem}

\begin{document}

\title{A Simple Efficiency Incremental Learning Framework via Vision-Language Model with Nonlinear Multi-Adapters}
%A Simple Efficiency Incremental Learning Framework via Vision-Language Model with Nonlinear Multi-Adapters
%Simple Efficiency Incremental Learning via Vision-Language Model with Multi-Adapters
%\author{
\author{Haihua Luo*, Xuming Ran*, Jiangrong Shen, Timo Hämäläinen, Zhonghua Chen, Qi Xu, Fengyu Cong
% \author{First A. Author, \IEEEmembership{Fellow, IEEE}, Second B. Author, and Third C. Author, Jr., \IEEEmembership{Member, IEEE}

% \thanks{This paragraph of the first footnote will contain the date on which you submitted your paper for review. It will also contain support information, including sponsor and financial support acknowledgment. For example, ``This work was supported in part by the U.S. Department of Commerce under Grant BS123456.'' }
% \thanks{The next few paragraphs should contain the authors' current affiliations, including current address and e-mail. For example, F. A. Author is with the National Institute of Standards and Technology, Boulder, CO 80305 USA (e-mail: author@boulder.nist.gov).}
% \thanks{S. B. Author, Jr., was with Rice University, Houston, TX 77005 USA. He is now with the Department of Physics, Colorado State University, Fort Collins, CO 80523 USA (e-mail: author@lamar.colostate.edu).}
% \thanks{T. C. Author is with the Electrical Engineering Department, University of Colorado, Boulder, CO 80309 USA, on leave from the National Research Institute for Metals, Tsukuba, Japan (e-mail: author@nrim.go.jp).}
% \thanks{This paragraph will include the Associate Editor who handled your paper.}}

\thanks{H. Luo, T. Hämäläinen and Z. Chen are with the Faculty of Information Technology, University of Jyväskylä, Jyväskylä, Finland (e-mail: {haihua.h.luo, timo.t.hamalainen, zhonghua.x.chen}@jyu.fi).}
\thanks{X. Ran is with National University of Singapore, Singapore, Singapore (e-mail: {ranxuming}@gmail.com).}
\thanks{J. Shen is with Lab of Brain-Machine Intelligence, Zhejiang University, Zhejiang, China (e-mail: {jiangrong}@zju.edu.cn).}
\thanks{Q. Xu and F. Cong are with School of Computer Science and Technology, Dalian University of Technology, Dalian, China (e-mail: {xuqi, cong}@dlut.edu.cn).}
\thanks{* Equal contribution}
\thanks{Correspondence to: Qi Xu ({xuqi}@dlut.edu.cn)}
}

% \markboth{Journal of IEEE Transactions on Artificial Intelligence, Vol. 00, No. 0, Month 2020}
% {First A. Author \MakeLowercase{\textit{et al.}}: Bare Demo of IEEEtai.cls for IEEE Journals of IEEE Transactions on Artificial Intelligence}

\maketitle

\begin{abstract}
Incremental Learning (IL) aims to learn new tasks while preserving previously acquired knowledge. Integrating the zero-shot learning capabilities of pre-trained vision-language models into IL methods has marked a significant advancement. However, these methods face three primary challenges: (1) the need for improved training efficiency; (2) reliance on a memory bank to store previous data; and (3) the necessity of a strong backbone to augment the model's capabilities. In this paper, we propose \textbf{SimE}, a \textbf{Sim}ple and \textbf{E}fficient framework that employs a vision-language model with adapters designed specifically for the IL task. We report a remarkable phenomenon: there is a nonlinear correlation between the number of adaptive adapter connections and the model's IL capabilities. While increasing adapter connections between transformer blocks improves model performance, adding more adaptive connections within transformer blocks during smaller incremental steps does not enhance, and may even degrade the model's IL ability. Extensive experimental results show that SimE surpasses traditional methods by 9.6\% on TinyImageNet and outperforms other CLIP-based methods by 5.3\% on CIFAR-100. Furthermore, we conduct a systematic study to enhance the utilization of the zero-shot capabilities of CLIP. We suggest replacing SimE's encoder with a CLIP model trained on larger datasets (e.g., LAION2B) and stronger architectures (e.g., ViT-L/14).

\end{abstract}

\begin{IEEEImpStatement}
Incremental learning (IL) is crucial for AI systems to adapt to new data without forgetting past knowledge, a challenge known as catastrophic forgetting. Current IL methods often suffer from high computational costs, reliance on large memory banks to store old data, and limited performance. This paper introduces SimE, a simple and efficient framework that leverages pre-trained vision-language models with lightweight adapters. SimE achieves state-of-the-art performance on standard benchmarks while requiring only a fraction of the trainable parameters and no memory bank. This breakthrough enables the development of more sustainable and scalable AI systems that can learn continuously in real-world scenarios, such as autonomous driving, medical diagnosis, and personalized robotics, with significantly reduced computational and memory footprints.
\end{IEEEImpStatement}


\section{Introduction}
\label{sec:intro}
Deep learning models have achieved significant success when fully trained on domain-specific tasks. However, in real-world scenarios, new data often come from diverse sources. Training a deep learning model on such new data typically lead to the model forgetting previously learned information—a phenomenon known as catastrophic forgetting \cite{goodfellow2013empirical}. To address this issue, Incremental Learning (IL) methods have been proposed, drawing inspiration from the human ability to learn continuously. These methods enable models to preserve existing knowledge while acquiring new skills \cite{de2021continual, masana2022class}. Traditional IL approaches, which start training from scratch \cite{li2017learning, serra2018overcoming, rebuffi2017icarl}, fail to leverage the zero-shot learning capabilities of pre-trained vision-language models. For example, Contrastive Language-Image Pre-training (CLIP) models \cite{radford2021learning}, trained on extensive datasets, exhibit strong feature extraction abilities.  Consequently, integrating CLIP's zero-shot learning capabilities into continual learning approaches has become a subject of keen interest \cite{thengane2022clip, ding2022don, zheng2023preventing, zhou2022learning, wang2023attriclip, yu2024boosting}.

Despite the success of recent CLIP-based IL methods, several challenges remain. For example, the CoOP framework \cite{zhou2022learning} preserves historical knowledge by utilizing a memory bank that is periodically accessed and updated during IL tasks. However, the ever-expanding volume of accumulated data can overburden the memory bank, thereby constraining CoOP's capacity for lifelong learning. In contrast, Continual CLIP \cite{thengane2022clip} leverages a frozen pre-trained CLIP to facilitate the model's continual learning capability, eliminating the need for replay memory but also limiting CLIP's zero-shot capabilities. Additionally, ZSCL \cite{zheng2023preventing} employs parameter regularization through knowledge distillation to maintain the model's performance across IL tasks. However, ZSCL is not entirely efficient for IL endeavors, as it requires a substantial number of finetuning parameters to learn new data features and demands significant GPU resources during training. Moreover, the efficacy of CLIP's feature extraction is significantly influenced by the pre-trained datasets and the size of its backbone architecture (e.g., ViT). Despite this, there has been no systematic exploration into optimizing CLIP's zero-shot learning potential, which also affects the performance of IL methods. Collectively, the main challenges faced by these approaches include: \textbf{(1) the need for enhanced training efficiency}; \textbf{(2) reliance on a memory bank to store previous data}; and \textbf{(3) the need for a robust backbone to enhance the model's capabilities.}

\begin{figure*}[t]
\centering
\includegraphics[width=0.99\textwidth]{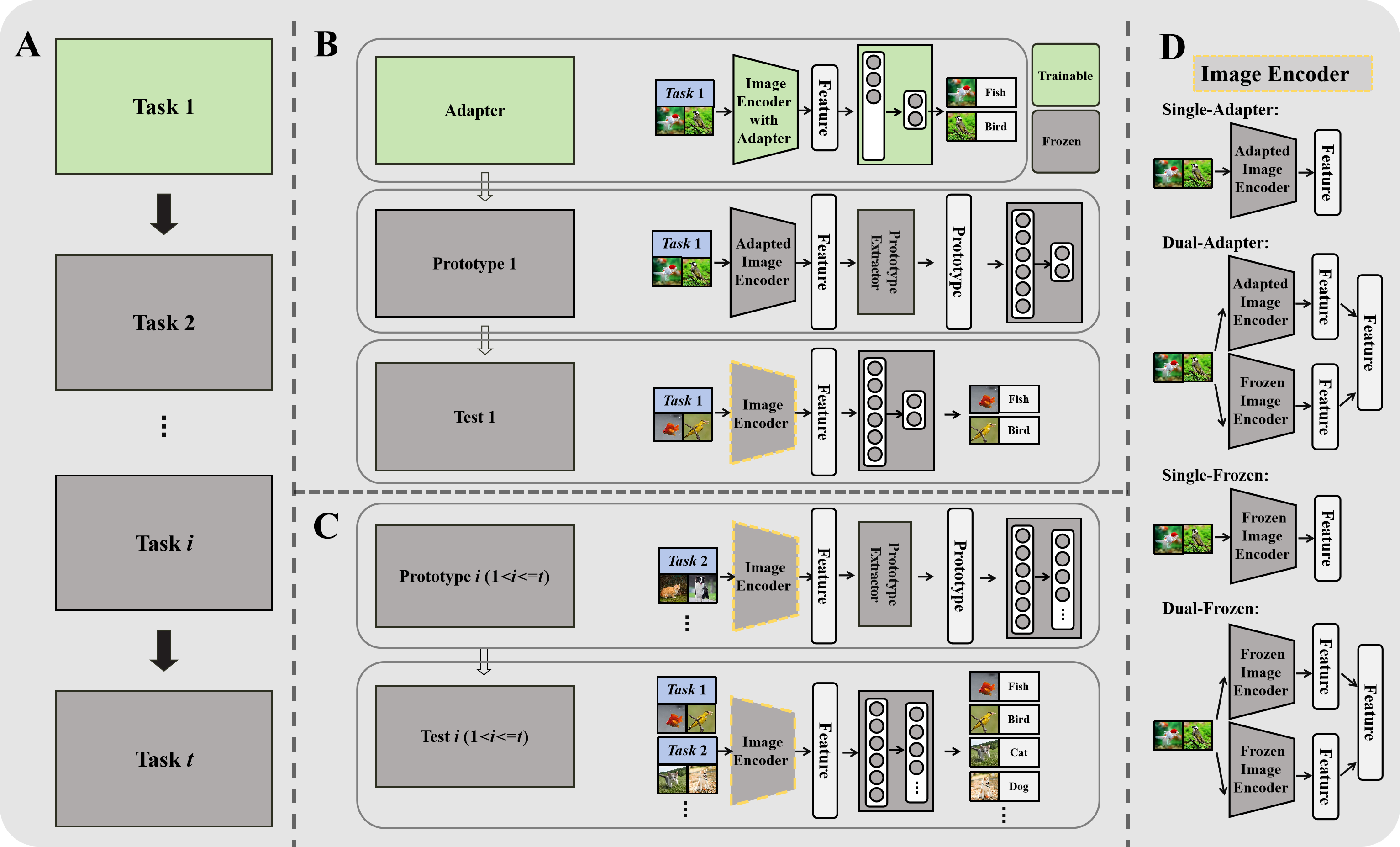}
\caption{\textbf{The overall framework of SimE.} The green represents trainable and the grey denotes frozen components. \textbf{A}) illustrates the incremental learning tasks, which include t tasks. Specifically, we finetune the trainable parameters in SimE for task 1, while freezing all the parameters in SimE for the remaining tasks.  \textbf{B}) The learning process for Task 1 can be divided into three stages: in the Adapter stage, the image encoder is finetuned using adapters; in the Prototype 1 stage, prototypes are computed based on the finetuned image encoder, and the classifier is updated; in the Test 1 stage, the classification performance of the model is evaluated. \textbf{C}) In the computation process for subsequent tasks $i$ ($1 < i < t$), all weights are frozen, only the prototypes are computed, and the classifier is updated. \textbf{D}) depicts the architectures of various image encoders.
}
\label{fig:framework}
\end{figure*}

To address these challenges, we introduce \textbf{SimE} (see Fig.\ref{fig:framework}), a Simple and Efficient IL framework that combines a vision-language model with an adapter designed for efficient IL tasks. Adapter \cite{houlsby2019parameter, chen2022adaptformer} is a lightweight module inserted into transformer blocks, enabling finetuning of the pre-trained model using minimal parameters. During training, the pre-trained model's parameters are frozen; we finetune only the adapter's trainable parameters, enhancing the model's parameter efficiency and adaptability without requiring a memory bank. We conduct a comprehensive evaluation of various backbones and pre-trained datasets to ascertain the most effective CLIP configurations for IL tasks using SimE. CLIP offers a spectrum of backbones, ranging from base to large models, as described by \cite{radford2021learning}, each with its own set of parameters. Additionally, the scope of pre-trained datasets is vast, as evidenced by works such as \cite{gadre2024datacomp} and \cite{cherti2023reproducible}, which span from 400 million to 2 billion samples. Our systematic investigation delves into the influence of these disparate backbones and pre-trained datasets on the performance of CLIP in IL scenarios.

\begin{figure*}[htb]
\centering
\includegraphics[width=1\textwidth]{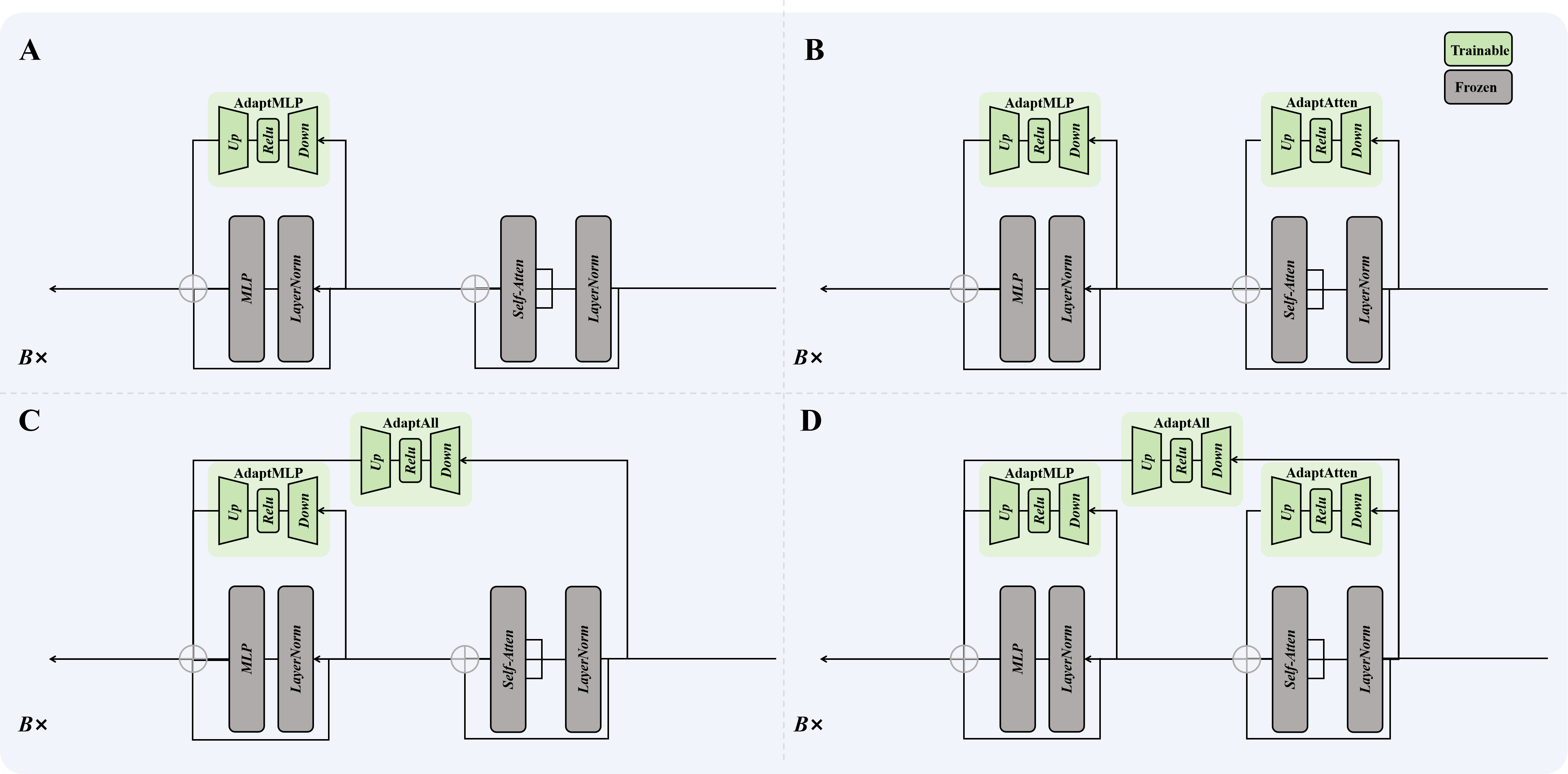}
\caption{\textbf{Comparison of previous and current finetuning approaches: The previous approach, AdaptFormer (A), is contrasted with our Multi-Adapter finetuning (B, C, and D).} The modules colored in green are trainable, while those in gray are frozen. In AdaptFormer and Multi-Adapter, the AdaptMLP, AdaptAtten, and AdaptAll modules are parameterized by a bottom-up bottleneck module with trainable parameters, whereas the original MLP and Self-Attention modules remain frozen. The AdaptFormer consists of the original frozen branch coupled with AdaptMLP. In contrast, our Multi-Adapter incorporates various trainable modules alongside the frozen branch for enhanced adaptability. And $B \times$ is represented by $B$ Blocks.}
\label{fig:ouradapter}
\end{figure*}

In SimE, by simply combining CLIP and AdaptFormer, we observe that increasing the number of adapters between transformer blocks can improve model performance. To explore better methods of adapter connections, we propose a new adapter design named Multi-Adapter (see Fig.\ref{fig:ouradapter}), which aims to increase the number of adaptive connections beyond the constraints imposed by the AdaptFormer architecture. Surprisingly, we find that within transformer blocks, increasing the number of adaptive connections in smaller incremental steps does not enhance, and may even degrade the model's IL capabilities. This improvement only occurs in larger incremental stages. Extensive experiments across various settings demonstrate the effectiveness of SimE on IL tasks. Our contributions can be summarized as follows:

\begin{itemize} 
\item We introduce SimE, which surpasses existing baseline IL models in class-incremental learning tasks. SimE is distinguished by its efficiency in three key areas: GPU usage, the number of trainable parameters, and memory size. Furthermore, SimE achieves competitive or superior accuracy with fewer additional parameters compared to other methods leveraging pre-trained models.
\item We propose Multi-Adapter to explore better methods of adapter connections and observe a significant phenomenon: there is a nonlinear correlation between the number of adaptive connections and the model's IL capabilities. While increasing the number of adapter connections between transformer blocks positively impacts model performance, within transformer blocks, adding more adaptive connections in smaller incremental steps does not enhance, and may even degrade the model's IL ability. Such improvement only occurs at more advanced incremental stages. 
\item We conduct a systematic study to enhance the utilization of the zero-shot capabilities of CLIP under SimE, pinpointing the most suitable backbone for CIFAR-100 and TinyImageNet. We advocate for the use of CLIP models that have been pre-trained on expansive datasets, such as LAION-2B, and possess larger architectures like ViT-L/14. 
\end{itemize}

\section{Related work}
\label{sec:related work}
\textbf{Conventional Continual Learning Methods:} Traditional continual learning methods can be divided into three categories: regularization-based, architecture-based, and replay-based approaches. Regularization-based methods \cite{li2017learning, aljundi2018memory, kirkpatrick2017overcoming} mitigate forgetting by incorporating regularization terms into the loss function, encouraging the model to retain weights important for previous tasks. However, these methods may diminish the model's ability to learn new categories effectively. Architecture-based methods \cite{serra2018overcoming, mallya2018packnet, wang2020learn} adjust the network's structure to accommodate new tasks by expanding it or altering its configuration. While effective, these methods may not be ideal for task-agnostic continual learning and can lead to increased memory usage. Replay-based methods \cite{rebuffi2017icarl, buzzega2020dark, cha2021co2l} involve storing and periodically revisiting data from previous tasks to help the model retain prior knowledge. Although useful, these methods can raise privacy concerns and may be less effective with smaller data buffers. Moreover, traditional continual learning models are typically trained from scratch, which may limit the maximum achievable performance by not leveraging pre-trained models.

\textbf{Continual Learning Methods Using CLIP:} Recently, pre-trained models have been increasingly adopted in continual learning due to their powerful feature extraction capabilities \cite{wang2022learning, wang2022dualprompt}. CLIP \cite{radford2021learning}, renowned for its impressive zero-shot abilities, excels in feature extraction through contrastive learning on vast amounts of image-text pairs. Consequently, several studies \cite{thengane2022clip, ding2022don, zheng2023preventing, zhou2022learning, wang2023attriclip, yu2024boosting} have integrated CLIP into continual learning models to enhance performance. Continual-CLIP \cite{thengane2022clip} directly applies CLIP to continual learning without any finetuning, maintaining CLIP's feature extraction capacity but potentially suffering from domain gaps between pre-trained datasets and downstream tasks. LwF-VR \cite{ding2022don} and ZSCL \cite{zheng2023preventing} finetune the entire model using traditional continual learning methods to adapt to specific tasks. This process is computationally expensive due to the large size of pre-trained models and may also lead to the forgetting of previously learned knowledge. Thus, the finetuning strategy significantly impacts model performance.

\textbf{Continual Learning Methods Using Adapter Finetuning:} Adapters were initially introduced in natural language processing \cite{houlsby2019parameter} to finetune pre-trained models for specific tasks by modifying a minimal set of weights. This approach has gained traction across various fields due to its notable efficiency \cite{chen2022adaptformer, dong2024efficient}. In the realm of continual learning, recent studies\cite{yu2024boosting, liu2023class, ermis2022memory, ermis2022continual} have explored integrating adapters, placing them after the encoder or within the model's blocks. These adapters enable learning new tasks with a limited number of trainable parameters while preserving the core feature extraction functions. AdaptFormer \cite{chen2022adaptformer}, known for its lightweight parameterization, enhances model efficiency but is limited by the number of adaptive connections it can establish. In this paper, we introduce a Multi-Adapter that expands the number of adaptive connections, thereby extending the model's flexibility.

\section{The SimE framework via vision-language models with adapters}
\label{sec:meth}
In this section, we first introduce the definition of the incremental learning (IL) task. Next, in Section \ref{sec:meth-SimE}, we present SimE, a framework that combines the image encoder in vision-language models with an adapter. Then, we introduce the formulation of the Multi-Adapter in Section \ref{sec:meth-SimE-multiadapter}. Finally, in Section \ref{sec:theorem introduction}, we introduce a theorem and a observation about adapter connection.

Consider a sequence of tasks $\mathcal{D} = {D_1, D_2, \ldots, D_T}$, where the $t$-th task is defined as $D_t = { (\mathbf{x}_i^t, \mathbf{y}_i^t)}_{i=1}^{m_t}$. Here, $D_t$ contains $m_t$ samples $\mathbf{x}_i^t$ and their corresponding labels $\mathbf{y}_i^t$. During the learning of task $D_t$, we have access only to the data from $D_t$; the data from previous tasks ${D_1, D_2, \ldots, D_{t-1}}$ are unavailable.

\subsection{SimE formulation}
\label{sec:meth-SimE}
SimE is structured into three primary phases: data pre-processing, feature extraction, and image classification. In the initial phase, raw image data are transformed into a format compatible with the model's requirements. This is followed by the feature extraction phase, where an encoder—specifically the image encoder from a pre-trained vision-language model equipped with an adapter and prototype extractors—processes the formatted images. The process culminates in the image classification phase, where a fully connected (FC) layer acts as the classifier. This classifier is intricately designed to support class-incremental learning (CIL), facilitating the seamless incorporation of new classes.

\textbf{The Encoder.}
Image encoders are utilized to extract visual features from preprocessed images. Commonly used image encoder architectures include ResNet and ViT. Taking ViT as an example, the $i$ th block of basic transformer module can be described as follows: 1) self-attention $f_i: \mathcal{X}_i \rightarrow \mathcal{A}_i$, which computes the attention among elements and learns global information through their interactions; 2) MLP $g_i: \mathcal{A}_i\rightarrow \mathcal{H}_i$, which applies nonlinear transformations to the input sequence to enhance the model's expressive capacity. Formally, this can be represented as:
\begin{equation}
\begin{aligned}
i\text{ th \ Self-Attention:} \ \bm{a}_i = f_i(\theta_i, \bm{x}_i),\\
i\text{ th \ MLP:} \ \bm{h}_i = g_i(\phi_i, \bm{a}_i).
\end{aligned}
\label{eq:frozenencoder}
\end{equation}

Here, $\bm{{h}}_i$ contains the visual features of the original image $\bm{x}_i$, Self-Attention and MLP are instantiated with the parameters $\theta_i$ and $\phi_i$ respectively. For the pre-trained encoder, both $\theta_i$ and $\phi_i$ are pre-trained weights that are frozen.

\textbf{The Adapter.} An adapter is a lightweight module designed to finetune pre-trained models for downstream datasets with a minimal number of additional parameters. The parameters of adapters are trainable and will be updated during the finetuning process, while the weights of pre-trained models are frozen. The adapter in the $i$ th blocks  extracts features as $d_i: \mathcal{X}_i \rightarrow {\mathcal{H}_i}$,
\begin{equation}
i\text{ th \ Adapter:} \ {\bm{p}_i} = d_i(\Tilde{\eta_i},\bm{x}_i).
\end{equation}
Here, adapter $i$ is instantiated with parameters $\Tilde{\eta_i}$, where $\Tilde{\eta_i}$ are trainable. The visual features extracted by the adapter are integrated into the pre-trained encoder, enhancing its ability to extract visual features of downstream datasets. It is noteworthy that, unlike the modules of the pre-trained encoder, both the number and positions of adapters are variable. By introducing adapter into pre-trained encoder, we get the general form of the encoder with adapter $E(\bm x)$:
% \begin{equation}
% E(\bm{x})= \sum_{i=1}^{B} (g_i(\phi_i, f_i(\theta_i,{\bm x}_i))+ d_i(\Tilde{\eta_i},f_i(\theta_i,{\bm x}_i)))
% % d_i(\Tilde{\eta_i},\bm{x}_i)),
% \label{adapter-overallV1}
% \end{equation}

\begin{equation}
E(\boldsymbol{x}) = \prod_{i=1}^B (g_i \left( \phi_i, f_i \left( \theta_i,  \bm{x}_i \right) \right)+ d_i(\Tilde{\eta_i},f_i(\theta_i,{\bm x}_i)))
\label{adapter-overallV1}
\end{equation}

where $B$ is the number of the block in the encoder. Especially, when $i=0$, $\bm{x}_i$ is the reprocessed image $\bm{x}$. More implementation of pre-trained model with adapters is described at Appendix \ref{appendixsec:meth-SimE-toymodel}.

\subsection{Multi-Adapter formulation}
\label{sec:meth-SimE-multiadapter}
We propose the Multi-Adapter, which comprises three adapter sub-modules: AdaptAtten, AdaptMLP, and AdaptAll as shown in Fig.\ref{fig:ouradapter}. The sub-modules AdaptAtten, AdaptMLP, and AdaptAll share the same structure, each containing a down-projection, a non-linear activation function (e.g., ReLU) and an up-projection. Thus, we can derive a more specific form of the $i$th adapter $r_{i}: \mathcal{C}_i \rightarrow \hat{\mathcal{S}}$,
\begin{equation}
{\bm{s}_{ij}} = r_{ij}(\Tilde{\eta_{ij}},\bm{c}_{ij})
\end{equation}
where
\begin{equation}
\bm{c}_{ij} = 
\begin{cases}
\bm{a}_{i}, & j=1, \ \bm{s}_{ij} = \bm{h}_{i}, \ \bm{r}_{ij} \ \text{is AdaptMLP}, \\
\bm{x}_{i}, & j=2, \ \bm{s}_{ij} = \bm{a}_{i}, \ \bm{r}_{ij} \ \text{is AdaptAtten}, \\
\bm{x}_{i}, & j=3, \ \bm{s}_{ij} = \bm{h}_{i}, \ \bm{r}_{ij} \ \text{is AdaptAll}.
\end{cases}
\end{equation}

here $\bm{c}_{ij}$ can be equal to the initial input $\bm{x}_{i}$ or an intermediate variable $\bm{a}_{i}$  in (\ref{eq:frozenencoder}), and $j \in \{1,2,3\}$, correspond to AdaptMLP, AdaptAtten, and AdaptAll, respectively.
Thus, the $i$ th block in ViT can be represented as a combination of the pre-trained modules and the adapter sub-modules:
% \begin{equation}
% E'(\bm{c})=\sum_{i}^{B} \sum_{j}^{Z}(g_{ij}(\theta_{ij}, f_{ij}(\phi_{ij},\bm{c}_{ij}))+ r_{ij}(\Tilde{\eta_{ij}},\bm{c}_{ij}))
% \label{adapter-overallV2}
% \end{equation}

\begin{equation}
E'(\bm{c})=\prod_{i=1}^B( \sum_{j}^{Z}(g_{ij}(\theta_{ij}, f_{ij}(\phi_{ij},\bm{c}_{ij}))+ r_{ij}(\Tilde{\eta_{ij}},\bm{c}_{ij})))
\label{adapter-overallV2}
\end{equation}

where $B$ is the number of the block in the encoder and $Z$ is a subset of $\{1,2,3 \}$ ($Z \subseteq \{1,2,3 \}$). Especially, when $i=0$, the $\bm{c}_{0j}$ is the reprocessed image $\bm{x}$. 
By identifying the trainable parameters, the adapter can be instantiated. Lastly, the optimisation of adapter for domain adaptation can see in Appendix \ref{appendixsec:optimisation-of-Adapter}.

\begin{table*}[htb]\scriptsize
\centering
\caption{Comparison on the accuracy of different CIL methods. The Average and Last accuracy of different CIL methods on CIFAR100 and TinyImageNet benchmark. The {\dag} indicates the result based on the CLIP ViT-L/14 pre-trained on Laion-2B. The 100 classes of TinyImageNet are used as base classes. The best results are coloured grey.}
\resizebox{1.0\textwidth}{!}{
\scriptsize
\begin{tabular}{l|cccccc|cccccc}
        \toprule
        & \multicolumn{6}{c|}{CIFAR100} & \multicolumn{6}{c}{TinyImageNet} \\
        & \multicolumn{2}{c}{10 Steps} & \multicolumn{2}{c}{20 Steps} & \multicolumn{2}{c|}{50 Steps}  & \multicolumn{2}{c}{5 Steps} & \multicolumn{2}{c}{10 Steps} & \multicolumn{2}{c}{20 Steps} \\
        Methods & Avg & Last & Avg & Last & Avg & Last & Avg & Last & Avg & Last & Avg & Last \\
        \midrule
        UCIR\cite{hou2019learning}     & 58.66 & 43.39 & 58.17 & 40.63 & 56.86 & 37.09 & 50.30 & 39.42 & 48.58 & 37.29 & 42.84 & 30.85 \\
        PASS\cite{zhu2021prototype}    & -	  & -	  &  -	  &   -	  &   -   &   -	  & 49.54 & 41.64 & 47.19 & 39.27 & 42.01 & 32.93 \\
        DyTox\cite{douillard2022dytox} & 67.33 & 51.68 & 67.30 & 48.45 & 64.39 & 43.47 & 55.58 & 47.23 & 52.26 & 42.79 & 46.18 & 36.21 \\
        DER\cite{yan2021dynamically}   & 74.64 & 64.35 & 73.98 & 62.55 & 72.05 & 59.76 &   -	  &   -   &   -	  &   -	  &   -	  &   -   \\
        \midrule
        CLIP\cite{radford2021learning}& 74.47	& 65.92	& 75.20	& 65.74	& 75.67	& 65.94	& 69.62	& 65.30	& 69.55	& 65.59	& 69.49	& 65.30 \\
        Fien-tune    	                        & 65.46	& 53.23	& 59.69	& 43.13	& 39.23	& 18.89	& 61.54	& 46.66	& 57.05	& 41.54	& 54.62	& 44.55 \\
        iCaRL\cite{rebuffi2017icarl}	        & 79.35	& 70.97	& 73.32	& 64.55	& 71.28	& 59.07	& 77.02	& 70.39	& 73.48	& 65.97	& 69.65	& 64.68 \\
        LwF\cite{li2017learning}	            & 65.86	& 48.04	& 60.64	& 40.56	& 47.69	& 32.90	& 60.97	& 48.77	& 57.60	& 44.00	& 54.79	& 42.26 \\
        Continual-CLIP\cite{thengane2022clip}   & 75.17 & 66.72 & 75.95 & 66.72 & 76.49 & 66.72 & 70.49 & 66.43 & 70.55 & 66.43 & 70.51 & 66.43 \\
        LwF-VR\cite{ding2022don}                & 78.81 & 70.75 & 74.54 & 63.54 & 71.02 & 59.45 & 77.56 & 70.89 & 74.12 & 67.05 & 69.94 & 63.89 \\
        ZSCL\cite{zheng2023preventing}    	    & 82.15	& 73.65	& 80.39	& 69.58	& 79.92	& 67.36	&80.27	& 73.57	& 78.61 & 71.62	& 77.18	& 68.30 \\
        SimE(Ours)	                            & 85.94	&77.10	&85.67	&76.61	&84.16	&73.88	& 79.35	&75.37	&79.32	&75.37	&79.29	&75.37 \\
        SimE(Ours)\dag	                        & \cellcolor{lightgray}91.66	& \cellcolor{lightgray}86.03	&\cellcolor{lightgray} 92.27	&\cellcolor{lightgray} 86.64	&\cellcolor{lightgray} 91.64	&\cellcolor{lightgray} 85.35	&\cellcolor{lightgray} 86.47	&\cellcolor{lightgray} 83.33 &\cellcolor{lightgray} 86.41	&\cellcolor{lightgray} 83.33	&\cellcolor{lightgray} 86.39	&\cellcolor{lightgray} 83.33 \\
        \bottomrule
\end{tabular}}
\label{table:main resluts}  
\end{table*}

\begin{table*}[htb]\scriptsize
\caption{Comparison between SimE and famous parameter-efficient CIL methods that utilize pre-trained models on multiple datasets. All datasets are split into 10 steps, except ObjectNet, which is split into 20 steps. All experiments are conducted based on CLIP ViT-B/16. The best results are coloured grey.}
\centering
\normalsize %\tiny \scriptsize \footnotesize \small \normalsize \large \Large
\renewcommand{\arraystretch}{1.2}
\resizebox{1.0\textwidth}{!}{
\begin{tabular}{l|cccccccccccccccc}
    \toprule
     &\multicolumn{2}{c}{CIFAR100} &\multicolumn{2}{c}{CUB200} &\multicolumn{2}{c}{ImaneNet-R} &\multicolumn{2}{c}{ImaneNet-100} &\multicolumn{2}{c}{VTAB} &\multicolumn{2}{c}{ImaneNet-A} &\multicolumn{2}{c}{ObjectNet} &\multicolumn{2}{c}{OmniBenchmark} \\
    Methods & Avg  & Last & Avg  & Last & Avg  & Last & Avg  & Last & Avg  & Last & Avg  & Last & Avg  & Last& Avg  & Last \\ 
    \midrule
     L2P\cite{wang2020learn} & 81.90 & 73.08 & 71.90 & 62.99 & 81.67 & 75.98 & 80.51 & 67.22 & 77.11 & 77.10 & 49.39 & 41.71 & 63.78 & 52.19 & 73.36 & 64.69\\
     DualPrompt\cite{wang2022dualprompt}   & 81.45 & 72.51 & 71.74 & 62.14 & 82.01 & 75.77 & 80.65 & 67.38 & 83.36 & 81.23 & 53.71 & 41.67 & 59.27 & 49.33 & 73.92 & 65.52\\
     CODA-Prompt\cite{smith2023coda}  & 76.98 & 62.25 & 66.61 & 50.88 & 78.00 & 67.52 & 64.13 & 34.76 & 83.90 & 83.02 & 53.54 & 42.73 & 66.07 & 53.29 & 77.03 & 68.09\\
     %SLCA\cite{zhang2023slca}  & 80.53 & 67.58 & 73.30 & 60.39 & 75.92 & 70.37 & 78.63 & 59.92 \\
     Boosting-CL\cite{yu2024boosting} & 85.21 & 77.52 & 69.56 & 55.09 & 83.48 & 74.81 & 88.05 & 80.90 & 76.69 & 66.25 & 61.00 & 49.04 & \cellcolor{lightgray}{68.12} & 55.30 & 82.13 & 73.88 \\
     APER\cite{zhou2025revisiting} & 75.76 & 65.50 & 78.80 & 70.61 & 78.62 & 71.35 & 85.84 & 76.40 & 85.95 & \cellcolor{lightgray}{84.35} & 60.47 & 49.37 & 67.18 & 55.24 & 80.75 & \cellcolor{lightgray}{74.37} \\
     MISA\cite{kang2025advancing} & 81.20 & \cellcolor{lightgray}{80.63} & 66.39 & 56.62 & 82.05 & 73.48 & 83.93 & 80.08 & 81.62 & 69.88 & 55.32 & 47.91 & 64.71 & 54.25 & 81.97 & 73.71 \\
     SimE(Ours)   & \cellcolor{lightgray}{85.94} & 77.10 & \cellcolor{lightgray}{84.98} & \cellcolor{lightgray}{76.68} & \cellcolor{lightgray}{83.19} & \cellcolor{lightgray}{75.82} & \cellcolor{lightgray}{89.77} & \cellcolor{lightgray}{80.94} & \cellcolor{lightgray}{86.69} & 78.04 & \cellcolor{lightgray}{63.19} & \cellcolor{lightgray}{50.56} & 67.51 & \cellcolor{lightgray}{55.74} & \cellcolor{lightgray}{83.29} & 74.19 \\
    \bottomrule
\end{tabular}}
\label{tab:sota}
\end{table*}

\subsection{Parameter spaces and function properties }
\label{sec:theorem introduction}
We next provide a theoretical analysis showing why inserting more Adapter blocks (or changing them) 
can increase the capacity of the model in a supremum sense, 
yet might not guarantee better actual performance under different insert settings.

Let $\Theta_0$ be the parameter space of the pre-trained encoder without adapters, then the model is $E_0(x)\in \Theta_0$.  Let $\Theta_{N,loc}$ be the parameter space when we insert $N$ adapter blocks at certain positions $loc$.
A performance measure is
\begin{equation}
f \bigl(x,\,\theta,\, N, \, loc\bigr)
\end{equation}
%\[f \bigl(x;\,\theta,\, N, \, \mathrm{loc}\bigr),\]
where $\theta\in \Theta_{N,loc}$ includes both original pre-trained parameters (often frozen) and newly introduced parameters of $N$ adapters.  We write
\begin{equation}
    \sup_{\theta\in \Theta_{N,\mathrm{loc}}} f(\theta)
\end{equation}
%\[\sup_{\theta\in \Theta_{N,\mathrm{loc}}} f(\theta)\]
for the maximum achievable performance in that space, and $f(\theta^*(\Theta_{N,\mathrm{loc}}))$ for the actual solution found by a given adapter location $loc$. Then we propose a theorem and put proof session at Appendix \ref{appendixsec:theorem proof}:

\begin{theorem}[Monotonic Bound on Performance]
\label{thm:1}

If
\[
  \Theta_0 \;\subseteq\; \Theta_{N,\mathrm{loc}} 
  \;\subseteq\; \Theta_{M,\mathrm{loc}},
  \quad M > N \ge 0,
\]
\text{then}
\begin{equation}
  \begin{aligned}
    \sup_{\theta\in \Theta_{M,\mathrm{loc}}} 
      f(x,\theta,M,\mathrm{loc})
    &\ge
    \sup_{\theta\in \Theta_{N,\mathrm{loc}}}
      f(x,\theta,N,\mathrm{loc}) \\
    &\ge
    \sup_{\theta\in \Theta_0}
      f(x,\theta,0,\varnothing).
  \end{aligned}
\end{equation}
\end{theorem}

Theorem.\ref{thm:1} tells us that, as we increase adapter number from $0$ to $N$ and further to $M$, we expand the parameter space and thus cannot lower the optimal performance. In practice, we expect the improvement of number of adapters could leverage this extra capacity, but no guarantee is given 
on whether the final solution $\theta^*$ is indeed close to the supremum.

\textit{Empirical Observation 1}

While Theorem.\ref{thm:1} shows that adding more adapters enlarges the parameter space and thus cannot lower the optimal performance in theory, we observed in our experiments that the actual performance does not always improves and sometimes even degrades when additional adapter connections are introduced within the same transformer block. For example, as shown in Table.\ref{table:adapter_structure}, on CIFAR-100 with 10 incremental steps, using both Adapt-MLP and Adapt-Atten (2.38M params) yields an average accuracy of 85.73\%, which is slightly lower than using Adapt-Atten alone (85.94\%, 1.19M params). Similar trends appear across 20-step settings. These indicate that expanding the hypothesis space does not guarantee that the optimization procedure finds a better solution, especially when the additional modules added on small incremental steps, and we refer to this phenomenon as the nonlinear relationship between adapter connections and model performance.

% \begin{theorem}[Non-monotonic Actual Solutions]
% \label{thm:2}
% Let $loc$ be a specific insertion strategy (AdaptMLP, AdapterAtten,AdapterAll).
% Then for $M>N$, it is possible that
% \begin{equation}
% \sup_{\theta\in \Theta_{M,\mathrm{loc}}} f(\theta)>\sup_{\theta\in \Theta_{N,\mathrm{loc}}} f(\theta)
% \end{equation}
% yet
% \begin{equation}
%   f\bigl(\theta_{M,\mathrm{loc}_1}^*\bigr)
%   \;<\;
%   f\bigl(\theta_{N,\mathrm{loc}_2}^*\bigr),
% \end{equation}
% showing that simply increasing adapter connections (e.g., MLP+Atten in the same block) does not necessarily beat a smaller or different-located Adapter design.
% \end{theorem}

% Theorem.\ref{thm:2} highlights a fundamental gap 
% between the theoretical capacity (Theorem.\ref{thm:1}) and 
% the actual performance of he model. In particular, adding more Adapters in a single block (transition from Atten to MLP+Atten, etc.) expands the space 
% yet can fail to yield better performance.

\begin{figure*}[htp] 
\centering
\begin{minipage}[b]{0.3\textwidth}
  \centering
  \includegraphics[width=\textwidth]{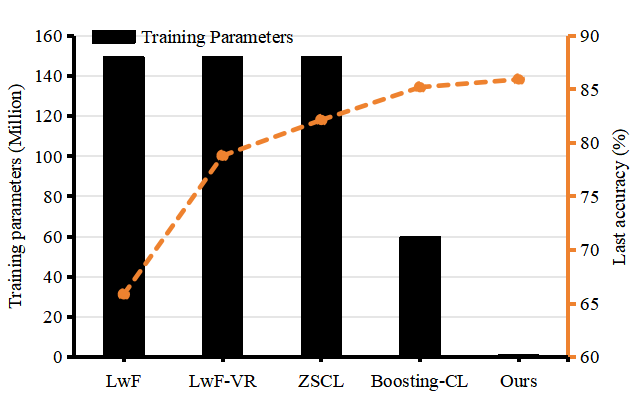}
  \text{(a)}
\end{minipage}
\begin{minipage}[b]{0.3\textwidth}
  \centering
  \includegraphics[width=\textwidth]{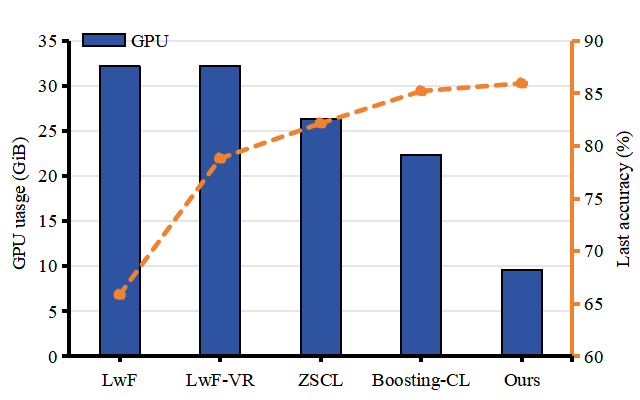}
  \text{(b)}
\end{minipage}
\begin{minipage}[b]{0.3\textwidth}
  \centering
  \includegraphics[width=\textwidth]{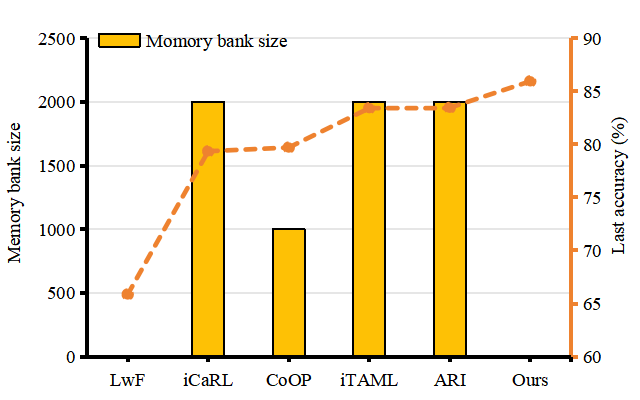}
  \text{(c)}
\end{minipage}

\vspace{-0.1em}

\begin{minipage}[b]{0.3\textwidth}
  \centering
  \includegraphics[width=\textwidth]{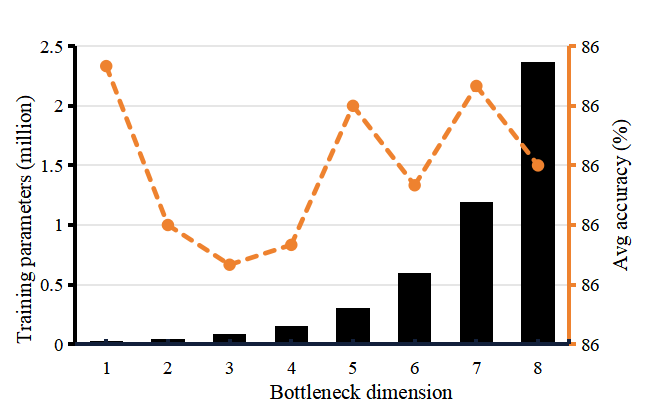}
  \text{(d)}
\end{minipage}
\begin{minipage}[b]{0.3\textwidth}
  \centering
  \includegraphics[width=\textwidth]{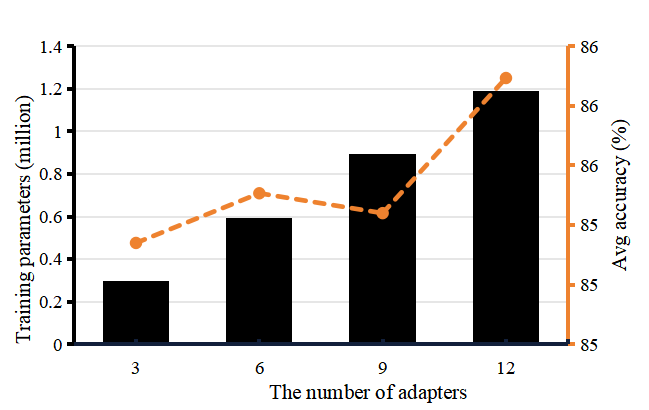}
  \text{(e)}
\end{minipage}
\begin{minipage}[b]{0.3\textwidth}
  \centering
  \includegraphics[width=\textwidth]{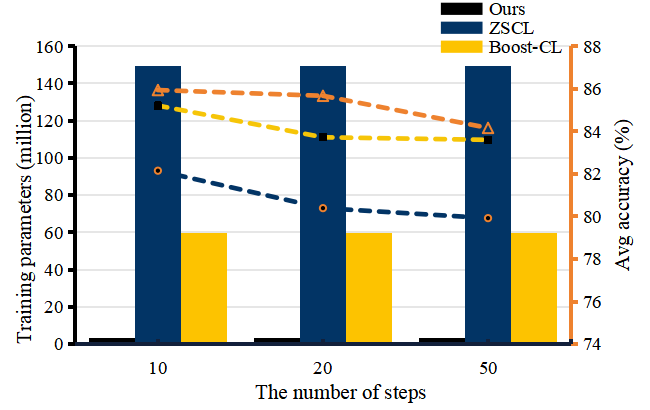}
  \text{(f)}
\end{minipage}
\vspace{-0.4em}
\caption{Comparison on the efficiency of different CIL methods. The dotted line and right axis coloured orange present the Last accuracy or Avg accuracy. (a)(b)(c) denote the Training parameters, GPU usage, Memory bank size and Last accuracy of different CIL methods respectively, (d)(e) is the Training parameters and Avg accuracy of Ours under different bottleneck dimensions and number of adapters. (f) show the comparison between Ours and other CIL methods in training parameters and Avg accuracy. All the experiments are conducted on CIFAR100, and (a)-(e) are conducted in 10steps.}
\label{fig:compeffenc}
\end{figure*}

\begin{figure*}[htp]
\centering
\begin{minipage}[b]{0.45\textwidth}
  \centering
  \includegraphics[width=\textwidth]{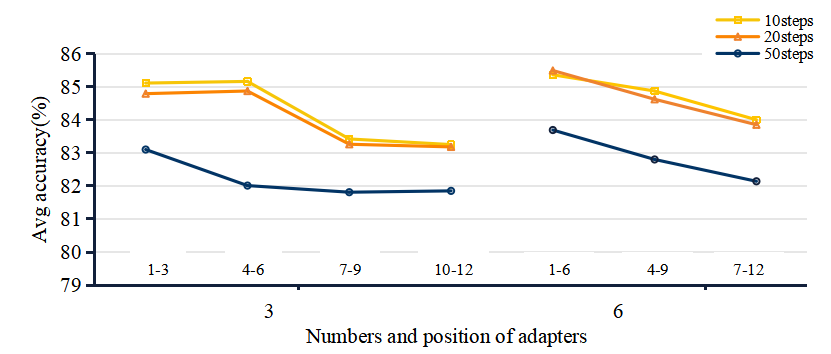}
  \text{(a)}
\end{minipage}
\begin{minipage}[b]{0.45\textwidth}
  \centering
  \includegraphics[width=\textwidth]{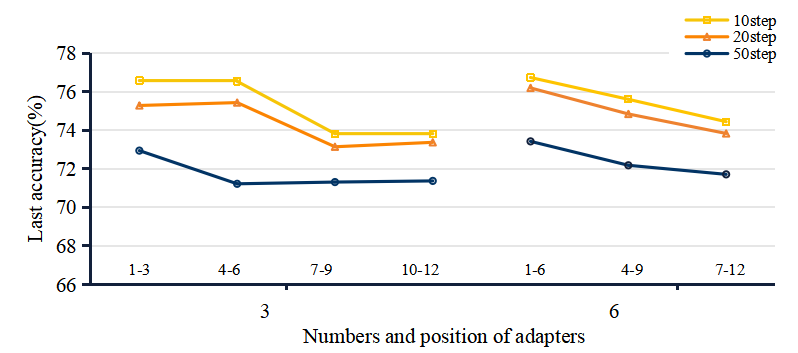}
  \text{(b)}
\end{minipage}
\begin{minipage}[b]{0.45\textwidth}
  \centering
  \includegraphics[width=\textwidth]{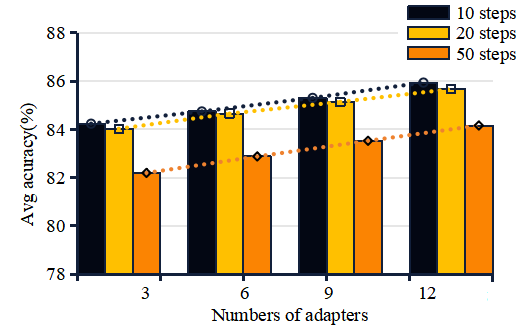}
  \text{(c)}
\end{minipage}
\begin{minipage}[b]{0.45\textwidth}
  \centering
  \includegraphics[width=\textwidth]{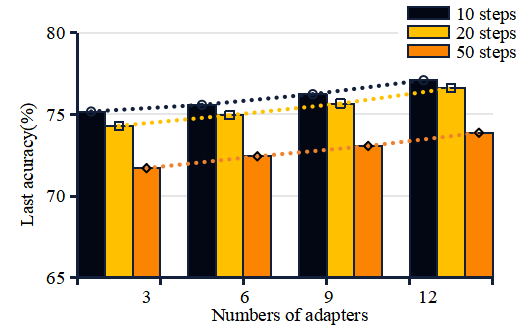}
  \text{(d)}
\end{minipage}
\caption{Influence of adapters' position and number between transformer blocks. The x-axis represents the number of adapters in the encoder, with the numerical ranges indicating the positions of the adapters. For example, "1-3" signifies that adapters are inserted in the first 3 blocks. The accuracy shown in (c) and (d) represents the average results for different adapter positions with the same number of adapters. All results are based on CIFAR100.}
\label{fig:numadapterwACC}
\end{figure*}

\section{Experimental results}
In this section, we begin by comparing the performance of the proposed SimE method with that of other Class-Incremental Learning (CIL) methods. Next, we evaluate the efficiency of these models by examining their number of trainable parameters, GPU usage, and memory bank size. Furthermore, we conduct ablation experiments to investigate the impact of various components within SimE. The details of the experimental settings are provided in Appendix \ref{appendixsec:exp-expsetting}.

\subsection{Comparison on the accuracy of different CIL methods}
\label{sec:exp-results}
First, we compare the performance of our method with other CIL methods across all tasks and present the results in Table.\ref{table:main resluts}. The results, show that our method achieves the highest scores among recent state-of-the-art methods, demonstrating the significant effectiveness of incorporating adapters into a pre-trained model. Specifically, our method and other CLIP-based methods (CoOP\cite{zhou2022learning}, Continual-CLIP\cite{thengane2022clip}) have a substantial advantage over traditional continual learning methods initially, reflecting the superior feature extraction capabilities of pre-trained models.

However, as shown in Appendix \ref{appendixsec:addition-acc-eff}, the accuracy of other CLIP-based methods drops quickly as training progresses, indicating that they are severely affected by domain gaps or catastrophic forgetting. Our method not only outperforms CLIP-based methods at the start, showing that finetuning helps the model adapt to downstream tasks, but also exhibits a slower decline in performance because it retains the original feature extractor, thus preserving the pre-trained model’s prior knowledge. 

We also compare our method with famous parameter-efficient CIL methods, like L2P\cite{wang2022learning} and DualPrompt\cite{wang2022dualprompt} on challenging datasets like ImageNet-R, and the results are presented in Table.\ref{tab:sota}. Our method remains superior in different datasets, further demonstrating its effectiveness.

Additionally, in Appendix \ref{appendixsec:addition-acc-eff}, we compare the performance of each task on CIFAR100 and TinyImageNet, revealing differences in the generalization ability of the backbone across different datasets.

\begin{table*}[htp]
\caption{The results of different implementations of Multi-Adapter. The structures of Adapter-MLP, Adapter-Atten and Adapt-All are shown in Fig.\ref{fig:ouradapter}. "Para" refers to trainable parameters, with "M" standing for million. All experiments are conducted on CIFAR100 and the best results are coloured grey}
\label{table:adapter_structure}
\centering
\tiny %\tiny \scriptsize \footnotesize \small \normalsize \large \Large
\renewcommand{\arraystretch}{0.8}
\resizebox{1.0\textwidth}{!}{
\begin{tabular}{ccc|c|cccccc}%{@{}p{2cm}p{1cm}p{1cm}p{1cm}p{1cm}p{1cm}p{1cm}@{}}
    \toprule
    & & & &\multicolumn{2}{c}{10 steps} &\multicolumn{2}{c}{20 steps} &\multicolumn{2}{c}{50 steps} \\ 
    \bf Adapt-MLP  & \bf Adapt-Atten  & \bf Adapt-All & \bf Para(M) & \bf Avg  & \bf Last & \bf Avg & \bf Last & \bf Avg & \bf Last \\ 
    \midrule
    \textcolor{lightgray}{\XSolidBrush} & \textcolor{lightgray}{\XSolidBrush} & \textcolor{lightgray}{\XSolidBrush} & 0 & 79.69 & 70.08 & 80.41 & 70.08 & 80.80 & 70.08 \\
    \CheckmarkBold & \textcolor{lightgray}{\XSolidBrush} & \textcolor{lightgray}{\XSolidBrush} & 1.19 & 85.60 & 76.70 & 85.30 & 76.02 & 84.09 & 73.77\\
    \textcolor{lightgray}{\XSolidBrush} & \CheckmarkBold & \textcolor{lightgray}{\XSolidBrush} & 1.19 & \cellcolor{lightgray}85.94 &\cellcolor{lightgray} 77.10 &\cellcolor{lightgray}85.67 &\cellcolor{lightgray}76.61 & 84.16 & 73.88\\
    \textcolor{lightgray}{\XSolidBrush} & \textcolor{lightgray}{\XSolidBrush} & \CheckmarkBold               & 1.19  & 85.77 & 76.83 & 85.48 & 76.16 & 84.16 & 73.86\\
     %\midrule
    \CheckmarkBold & \CheckmarkBold & \textcolor{lightgray}{\XSolidBrush} & 2.38 & 85.73 & 76.79 & 85.42 & 76.09 &84.76 &74.66\\
    \textcolor{lightgray}{\XSolidBrush} & \CheckmarkBold & \CheckmarkBold & 2.38 & 85.84 & 76.98 & 85.36 & 76.03 & 84.75 & 74.69\\
    \CheckmarkBold & \textcolor{lightgray}{\XSolidBrush} & \CheckmarkBold & 2.38 & 85.63 & 76.65 & 85.12 & 75.66 & 84.75 & 74.76\\
    %\midrule
    \CheckmarkBold & \CheckmarkBold & \CheckmarkBold & 3.57 & 85.54 & 76.51 & 85.05 & 75.53 & \cellcolor{lightgray}85.00 & \cellcolor{lightgray}75.16\\
    \bottomrule
\end{tabular}}
\end{table*}

\begin{table*}[htb]\scriptsize
\caption{The influence of CLIP pre-trained datasets. WIT-400M is the closed-source dataset of OpenAI while others are from Open\_CLIP. All results are conducted on ViT-B/16 and the 100 classes of TinyImageNet are used as base classes. The best results are coloured grey.}
\label{tab:pre-trained datasets}
\centering
\resizebox{1.0\textwidth}{!}{
\begin{tabular}{l|cccccc|cccccc}
    \toprule
    & \multicolumn{6}{c|}{CIFAR100} & \multicolumn{6}{c}{TinyImageNet} \\
    &\multicolumn{2}{c}{10 steps} &\multicolumn{2}{c}{20 steps} &\multicolumn{2}{c|}{50 steps}  &\multicolumn{2}{c}{10 steps} &\multicolumn{2}{c}{20 steps} &\multicolumn{2}{c}{50 steps}\\ 
    \bf Blocks & \bf Avg  & \bf Last & \bf Avg & \bf Last & \bf Avg & \bf Last  & \bf Avg  & \bf Last & \bf Avg & \bf Last & \bf Avg & \bf Last \\ 
    \midrule
    WIT-400M\cite{radford2021learning}         & 85.60 & 76.70 & 85.30 & 76.02 & 84.09 & 73.77  & 79.35 & 75.37 & 79.32 & 75.37 & 79.29 & 75.37 \\ 
    Laion-400M\cite{schuhmann2022laion}       & 87.14 & 79.54 & 86.86 & 78.82 & 85.95 & 77.63  & 80.62 & 78.01 & 80.46 & 77.48 & 81.06 & 79.22 \\
    Laion-2B\cite{schuhmann2022laion}         &\cellcolor{lightgray}88.34 &\cellcolor{lightgray}81.33 &\cellcolor{lightgray}88.47 &\cellcolor{lightgray}80.89 &\cellcolor{lightgray}87.91 &\cellcolor{lightgray}80.09  &\cellcolor{lightgray}81.98 &\cellcolor{lightgray}79.99 &\cellcolor{lightgray}81.83 &\cellcolor{lightgray}79.77 &\cellcolor{lightgray}82.78 &\cellcolor{lightgray}81.63 \\
    DataComp-1B\cite{gadre2024datacomp}      & 88.04 & 80.77 & 87.89 & 79.88 & 87.57 & 79.25  & 81.38 & 79.11 & 81.49 & 78.70 & 82.51 & 80.89 \\
    CommonPool-1B\cite{gadre2024datacomp}    & 86.96 & 78.74 & 86.58 & 77.88 & 86.21 & 77.10  & 80.13 & 76.95 & 80.26 & 76.69 & 81.24 & 78.70 \\
    \bottomrule
\end{tabular}}
\end{table*}

\subsection{Comparison on the efficiency of different CIL methods }
We compare the efficiency of our proposed SimE method with other Class-Incremental Learning (CIL) methods by examining the number of trainable parameters, GPU usage, and replay data size. The experimental settings are the same as those described in Appendix \ref{appendixsec:exp-expsetting} and the results are shown in Fig.\ref{fig:compeffenc}. As illustrated in Fig.\ref{fig:compeffenc}(a), our method requires only thousands of trainable parameters while achieving competitive results compared to other CIL methods that utilize millions of parameters, significantly reducing training costs. Furthermore, as shown in Fig.\ref{fig:compeffenc}(b) and Fig.\ref{fig:compeffenc}(c), our method uses only one-third of the parameters and does not require a buffer to store replay data. 

We also study the influence of the bottleneck dimension and the number of adapters, as depicted in Fig.\ref{fig:compeffenc}(d) and Fig.\ref{fig:compeffenc}(e). Despite varying these parameters, our method still achieves competitive performance with minimal trainable parameters. These results demonstrate that our method can achieve performance comparable to or even exceeding that of other CIL methods with minimal training costs, thereby strongly validating the efficiency of the proposed SimE method.

\subsection{Ablation study on the influence of adapter components}

\textbf{The influence of adapter connection between transformer blocks.} We investigated the impact of the position and number of adapters inserted between transformer blocks, presenting the results in Fig.\ref{fig:numadapterwACC}, more details are presented in Appendix \ref{appendixsec:add-adapter-abla}. From Fig.\ref{fig:numadapterwACC}, it is evident that inserting the same number of adapters into the earlier blocks significantly improves model performance. This suggests that learning primary features plays a more crucial role in model finetuning. Additionally, we varied the number of adapters between transformer blocks from 1 to 12. Inserting adapters into every block (12 adapters in total) consistently yielded the best performance across all steps. Therefore, a larger number of adapters between transformer blocks leads to better model performance, indicating that increasing the number of adapter connections between transformer blocks positively impacts model outcomes.

\textbf{The influence of adapter connection within transformer blocks.} We also tested different implementations of the Multi-Adapter by inserting adapters within all 12 transformer blocks. The results are reported in Table.\ref{table:adapter_structure}, and more details are presented in Appendix \ref{appendixsec:add-adapter-abla}. Interestingly, we found that in smaller incremental steps, increasing the number of adaptive connections within transformer blocks does not improve model performance; in fact, it can even degrade it. The previously observed positive correlation only occurs in larger incremental steps. This suggests there is a nonlinear correlation between the number of adaptive adapter connections and the model's IL capabilities. This phenomenon can be attributed to the fact that, in smaller incremental steps, the model faces relatively mild distribution shifts between tasks, and excessive adaptive connections may lead to overfitting to the current task, thereby harming knowledge retention from previous tasks. Moreover, adding multiple adapters within the same block can cause redundancy in learned representations, reducing the effective utilization of model capacity. In contrast, for larger incremental steps with more significant distribution changes, the additional adaptive connections provide the flexibility needed to capture new task-specific features, thus improving performance. In SimE, we explore the optimal implementation of adapters across various task configurations.

\subsection{Ablation studies on the influence of CLIP components}
\label{sec:Sasdasd}
\textbf{The influence of pre-trained datasets.} CLIP has attracted significant attention due to its powerful zero-shot capabilities, leading many studies to retrain CLIP from scratch on other image-text pair datasets, such as Datacomp \cite{gadre2024datacomp} and LAION \cite{schuhmann2022laion}. These datasets are comparable or even larger than the original pre-training dataset (WIT-400M). We evaluated their performance on CIFAR-100 and TinyImageNet as shown in Table.\ref{tab:pre-trained datasets}, we also employ t-SNE\cite{van2008visualizing} to visualize classification at the beginning stage shown in Fig.\ref{fig:appendix-tsne}. As shown in Fig.\ref{fig:appendix-tsne}, larger and more diverse pre-training datasets (e.g., LAION-2B) tend to produce embeddings with clearer inter-class boundaries and tighter intra-class clusters, which helps downstream classification. In contrast, models pre-trained on smaller or less diverse datasets may yield overlapping clusters, increasing the difficulty of learning new classes without forgetting. This visual evidence supports the results in Table.\ref{tab:pre-trained datasets} and highlights the dataset scale and diversity in improving IL performance.

\begin{figure}[htb]
\centering
    \begin{minipage}{0.225\textwidth}
        \includegraphics[width=\linewidth]{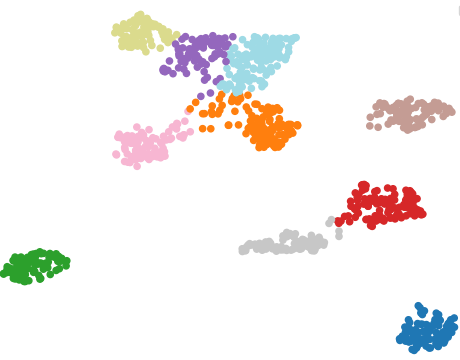}
        \centering {\scriptsize(a)Laion400M}
    \end{minipage}
    \hspace{0.02\textwidth}
    \begin{minipage}{0.225\textwidth}
        \includegraphics[width=\linewidth]{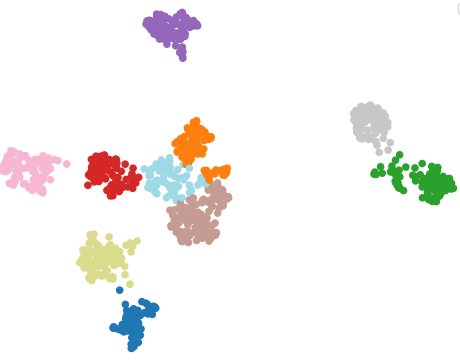}
        \centering {\scriptsize(b)Laion2B}
    \end{minipage}
    \begin{minipage}{0.225\textwidth}
        \includegraphics[width=\linewidth]{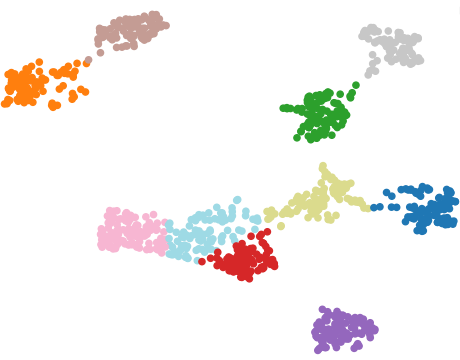}
        \centering {\scriptsize(c)Datacomp1B}
    \end{minipage}
    \hspace{0.02\textwidth}
    \begin{minipage}{0.225\textwidth}
        \includegraphics[width=\linewidth]{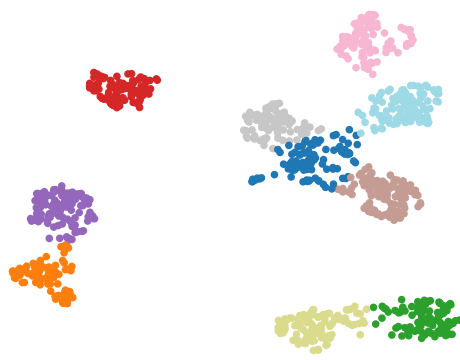}
        \centering {\scriptsize(d)CommonPool1B}
    \end{minipage}
% \vspace{-0.3em}
\caption{The t-SNE visualization of CLIP pre-trained on different datasets. All results are conducted on CIFAR100 with ViT-B/16 as backbone containing 10 classes.}
\label{fig:appendix-tsne}
\end{figure}

\begin{table}[htb]
\caption{The influence of CLIP ViT backbones size. The experiments use the corresponding data preprocessing while "336px" indicates the images are resized to 336. All experiments conducted on CIFAR100 and the best results are colored grey.}
\label{tab:ViT}
\centering
\footnotesize %\tiny \scriptsize \footnotesize \small \normalsize \large \Large
\renewcommand{\arraystretch}{1}
\begin{tabular}{p{2.2cm}p{0.6cm}p{0.6cm}p{0.6cm}p{0.6cm}p{0.6cm}p{0.6cm}}
\toprule
    &\multicolumn{2}{c}{10 steps} &\multicolumn{2}{c}{20 steps} &\multicolumn{2}{c}{50 steps} \\ 
    \bf Blocks & \bf Avg  & \bf Last & \bf Avg & \bf Last & \bf Avg & \bf Last \\ 
    \midrule
    ViT-B/16         & 85.94 & 77.10 & 85.67 & 76.61 & 84.16 & 73.88 \\ 
    ViT-B/32         & 83.60 & 74.43 & 82.02 & 71.74 & 81.18 & 70.06 \\
    ViT-L/14-336px   & 88.53 & 80.85 & 88.02 & 80.45 & 89.12 & 81.65 \\
    ViT-L/14         &\cellcolor{lightgray}88.79 &\cellcolor{lightgray}81.44 &\cellcolor{lightgray}88.57 &\cellcolor{lightgray}81.01 &\cellcolor{lightgray}89.73 &\cellcolor{lightgray}{82.60} \\
    \bottomrule
  \end{tabular}
\end{table}

\textbf{The influence of  ViT backbone size.} In addition to examining pre-trained datasets, we also investigated the impact of different ViT backbones in CLIP. Our default model uses ViT-B/16. As shown in Table.\ref{tab:ViT}, increasing the backbone size significantly improves model performance. Specifically, the accuracy of ViT-L consistently surpasses that of ViT-B across various settings, demonstrating superior feature extraction capabilities. When the model size is held constant, the size of the image patches plays a crucial role in feature extraction, with smaller patch sizes better capturing semantic information. In contrast, the size of the images during pre-processing has a relatively minor impact on model performance.

\section{Conclusion}
In this paper, we propose SimE, a simple yet efficient incremental learning (IL) framework. SimE utilizes a pre-trained model as the encoder and incorporates adapters for finetuning, thereby achieving robust feature extraction capabilities while adapting to IL tasks without the need to store replay data. Our experiments demonstrate that SimE achieves competitive results, validating its effectiveness. To explore better methods of adapter connections, we introduce the Multi-Adapter and observe a remarkable phenomenon: there is a nonlinear correlation between the number of adaptive adapter connections and the model's IL capabilities. Specifically, while increasing the number of adapter connections between transformer blocks positively impacts model performance, adding more adaptive connections within transformer blocks during small incremental steps does not enhance, and may even degrade the model's IL ability. Such improvements occur only at more advanced incremental stages. We also conducted a systematic study  and recommend that SimE's backbone encoder utilize the image encoder from CLIP models pre-trained on larger datasets like LAION-2B and larger architectures such as ViT-L/14 for CIL tasks. In future work, we will explore combining SimE with different pre-trained large models and various types of adapters for other tasks.

%----------------------Bibliography---------------------------
%\clearpage

%----------------------Appendix---------------------------
\clearpage
\appendices
\begin{center}
    \LARGE\textbf{Appendix}
\end{center}

In the Appendix, we provide additional method details and experiment settings mentioned in the main text, as well as experiment results. The contents of supplementary material are organized as follows:
\begin{itemize}
    \item \textbf{Section \ref{appendixsec:meth-SimE-toymodel}} introduce the specific details of the prototype extractor and classifier of SimE, and the realizations of SimE with adapter. We extract the average value of features as prototypes of classes and calculate the similarity between prototypes and features to classify test data. In realizations of SimE, we first build a toy model of SimE and then explore various specific implementations of SimE.
    \item \textbf{Section \ref{appendixsec:optimisation-of-Adapter}} describes the update process of the encoder. The model starts with a pre-trained CLIP encoder, and after finetuning in the Adapter stage, a finetuned encoder is obtained. To maximize the feature extraction capability of the pre-trained model, we also include an additional pre-trained encoder. Together, these two encoders form the encoder of the model.
    \item \textbf{Section \ref{appendixsec:theorem proof}} present the proof session of theorem 1. The parameters are defined as in Sec \ref{sec:theorem introduction}.
    \item \textbf{Section \ref{appendixsec:exp-expsetting}} outlines the various experimental settings, including datasets, model backbone, evaluation metrics, CIL methods for comparison, and the training configurations of the model. In the paper, unless otherwise specified, the experimental setup should be consistent with what is described here.
    \item \textbf{Section \ref{appendixsec:addition-acc-eff}} presents a comparison of our method with other CIL methods in terms of accuracy and efficiency, including accuracy comparisons across various tasks and between different datasets. This serves as a supplement to the experimental results in the main paper.
    \item \textbf{Section \ref{appendixsec:add-adapter-abla}} supplements the results of ablation experiments, including between transformer blocks and within blocks. In these experiments, "CLIP" represents the results when only the pre-trained encoder is used without an adapter.
\end{itemize}

\begin{table*}[htbp]\scriptsize
\centering
\caption{Comparison on the accuracy of different CIL methods. The Average and Last accuracy of different CIL methods on CIFAR100 and ImageNet-100 benchmark. All the methods use CLIP ViT-B/16 as the backbone. In SimE, all results use 12 adapters (one per transformer block). The best results are coloured grey.}
\resizebox{1.0\textwidth}{!}{
\scriptsize
\begin{tabular}{l|cccccc|cccccc}
        \toprule
        & \multicolumn{6}{c|}{CIFAR100} & \multicolumn{6}{c}{ImageNet-100} \\
        & \multicolumn{2}{c}{10 Steps} & \multicolumn{2}{c}{20 Steps} & \multicolumn{2}{c|}{50 Steps}  & \multicolumn{2}{c}{5 Steps} & \multicolumn{2}{c}{10 Steps} & \multicolumn{2}{c}{20 Steps} \\
        Methods & Avg & Last & Avg & Last & Avg & Last & Avg & Last & Avg & Last & Avg & Last \\
        \midrule
        UCIR\cite{hou2019learning}    	& 58.70	& 42.90	& 58.20	& 41.10	&  -	& -	    & 76.00	& 64.00	& 70.50	& 55.30	& 64.70	& 47.80\\
        iCaRL\cite{rebuffi2017icarl}	& 79.35	& 70.97	& 73.32	& 64.55	& 71.28	& 59.07	& -  	& -	    & -	    & -	    & -	    & - \\
        DyTox\cite{douillard2022dytox} & 77.00 & 67.50	& 76.80	& 64.30	& 75.50	& 59.50	& -	    & -	    & 80.80	& 72.50	& -	    & -    \\
        DER\cite{yan2021dynamically}   & 74.64 & 64.35 & 73.98 & 62.55 & 72.05 & 59.76 &   -   &   -   & 76.10 & 66.10 &   -	&   -  \\
        SimE(Ours)	                    &\cellcolor{lightgray} 85.94	&\cellcolor{lightgray} 77.10	&\cellcolor{lightgray} 85.67	&\cellcolor{lightgray} 76.61	&\cellcolor{lightgray} 84.16	&\cellcolor{lightgray} 73.88	&\cellcolor{lightgray} 79.35	&\cellcolor{lightgray} 75.37	&\cellcolor{lightgray} 79.32	&\cellcolor{lightgray} 75.37	&\cellcolor{lightgray} 79.29	&\cellcolor{lightgray} 75.37\\
        \bottomrule
\end{tabular}}
\label{table:same_backbone}  
\end{table*}

\begin{table*}[htbp]\scriptsize
\centering
\caption{The Average and Last accuracy on CIFAR100 and ImageNet-100 benchmark under different classes imbalance ratio $imb\_factor$. $imb\_factor=1$ means every classes have same number of samples. The best results are coloured grey.}
\resizebox{1.0\textwidth}{!}{
\scriptsize
\begin{tabular}{l|cccccc|cccccc}
        \toprule
        & \multicolumn{6}{c|}{CIFAR100} & \multicolumn{6}{c}{ImageNet-100} \\
        & \multicolumn{2}{c}{10 Steps} & \multicolumn{2}{c}{20 Steps} & \multicolumn{2}{c|}{50 Steps}  & \multicolumn{2}{c}{5 Steps} & \multicolumn{2}{c}{10 Steps} & \multicolumn{2}{c}{20 Steps} \\
        $imb\_factor$ & Avg & Last & Avg & Last & Avg & Last & Avg & Last & Avg & Last & Avg & Last \\
        \midrule
        0.01 & 80.73 & 71.12 & 81.42 & 70.05 & 77.10 & 65.64 & 87.70 & 79.96 & 88.82 & 79.68 & 87.87 & 78.74 \\
        0.05 & 84.21 & 75.50 & 83.77 & 73.97 & 80.38 & 69.22 & 88.58 & 80.58 & 89.03 & 80.18 & 88.63 & 79.18\\
        0.1  & 84.70 & 75.98 & 84.10 & 74.57 & 80.73 & 69.35 & 88.92 & 81.04 & 89.32 & 80.50 & 89.28 & 79..82\\
        0.5  & 85.71 & 76.87 & 85.66 & 76.44 & 83.00 & 72.41 & 89.23 & 81.28 & 89.60 & 80.68 & 89.86 & 80.60\\
        1    &\cellcolor{lightgray} 85.94 &\cellcolor{lightgray} 77.10 &\cellcolor{lightgray} 85.67 &\cellcolor{lightgray} 76.61 &\cellcolor{lightgray} 84.16 &\cellcolor{lightgray} 73.88 &\cellcolor{lightgray} 89.31 &\cellcolor{lightgray} 81.32 &\cellcolor{lightgray} 89.77 &\cellcolor{lightgray} 80.94 &\cellcolor{lightgray} 89.90 &\cellcolor{lightgray} 80.64 \\
        \bottomrule
\end{tabular}}
\label{table:class_imbalance}  
\end{table*}

\begin{table*}[htbp] \scriptsize
\caption{Last accuracy of different CIL methods on CIFAR100. The accuracy of Task $t$, $t \in \{1, 2, \ldots, 10 \}$ reported here is the last accuracy over all the previous tasks (i.e., Tasks $1, 2, \ldots, t$). If not otherwise specified, the method uses ResNet as the backbone, where $\dag$ indicates the result based on the CLIP ViT-L/14 pre-trained on Laion-2B. The best results are coloured grey.}
\label{table:Appendixcompacc10stepscafir100}
\centering
\resizebox{1.0\textwidth}{!}{
\begin{tabular}{l|cccccccccc}
  \toprule  
  Method & Task 1  &  Task 2  &  Task 3  & Task 4  & Task 5  & Task 6  & Task 7  & Task 8  &  Task 9  & Task 10  \\
  \hline  
  LwF\cite{li2017learning}      & 89.3 & 70.1 & 54.3 & 45.8 & 39.8 & 36.1 & 31.7 & 28.9 & 24.4 & 23.9 \\
  iCaRL\cite{rebuffi2017icarl}    & 88.7 & 78.1 & 72.4 & 67.2 & 63.7 & 60.2 & 56.4 & 54.4 & 51.9 & 49.5 \\
  iTAML\cite{rajasegaran2020itaml}   & 89.2 & 89.0 & 87.3 & 86.2 & 84.3 & 82.1 & 80.7 & 79.1 & 78.4 & 77.8 \\
  ARI\cite{wang2022anti}     & 88.6 & 86.9 & 85.8 & 84.6 & 83.1 & 81.8 & 81.6 & 81.0 & 80.2 & 80.9 \\
  \midrule 
  CoOp(W ViT-L/14)\cite{zhou2022learning}    & 95.8 & 90.7 & 85.2 & 83.4 & 80.8 & 75.8 & 74.7 & 71.7 & 71.3 & 67.6 \\
  Continual-CLIP(ViT-L/14)\cite{thengane2022clip}  & 96.7 & 92.2 & 86.0 & 80.4 & 77.5 & 75.8 & 73.0 & 71.4 & 69.8 & 66.7 \\
  AttriCLIP(ViT-L/14)\cite{wang2023attriclip}  & 97.8 & 93.7 & 91.0 & 87.5 & 84.7 & 82.5 & 82.3 & 81.9 & 81.7 & 81.4 \\
  \midrule 
  \bf Ours(ViT-B/16 \& AdaptMLP)   & 97.1 & 94.4 & 90.4 & 87.9 & 86.1 & 84.0 & 82.0 & 79.5 & 78.0 & 76.7\\
  \bf Ours(ViT-L/14 \& AdaptMLP)   & 98.2 &95.6  &91.9  & 90.2 & 89.5 & 87.9 & 85.6 & 83.5 & 82.7 & 81.3\\
  \bf Ours(ViT-B/16 \& AdaptAtten) & 97.0 & 94.5 & 90.7 & 88.4 & 86.7 & 84.6 & 82.4 
  & 79.8 & 78.3 & 77.1\\
  \bf Ours(ViT-L/14 \& AdaptAtten) & 98.3 & 95.9 & 92.0 & 90.2 & 89.6 & 88.0 & 85.8 
  & 83.7 & 82.9 & 81.4\\
 \textbf{Ours$^{\dag}$ (ViT-L/14 \& AdaptAtten \& Laion-2B)}   & \cellcolor{lightgray}98.2 
  & \cellcolor{lightgray}96.9 & \cellcolor{lightgray}94.6 & \cellcolor{lightgray}93.1 & \cellcolor{lightgray}92.5 & \cellcolor{lightgray}91.1 
  & \cellcolor{lightgray}89.4 & \cellcolor{lightgray}87.7 & \cellcolor{lightgray}87.2 & \cellcolor{lightgray}86.0\\
  \toprule  
\end{tabular}}
\end{table*}

\begin{figure*}[htbp]
\centering
\begin{minipage}[b]{0.3\textwidth}
  \centering
  \includegraphics[width=\textwidth]{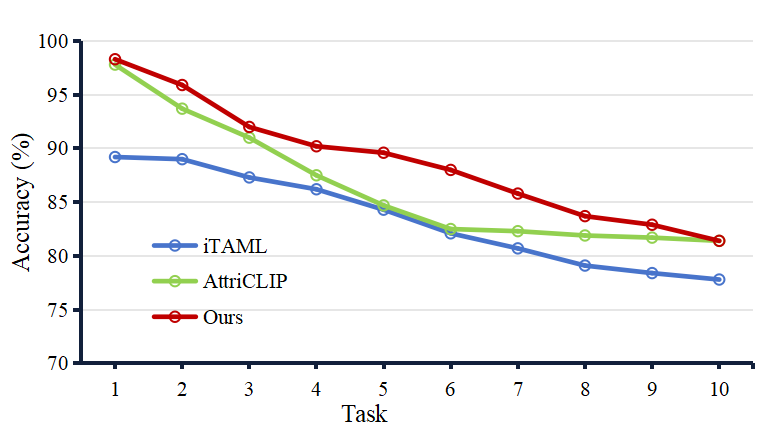}
  %\text{(a)}
\end{minipage}
\begin{minipage}[b]{0.3\textwidth}
  \centering
  \includegraphics[width=\textwidth]{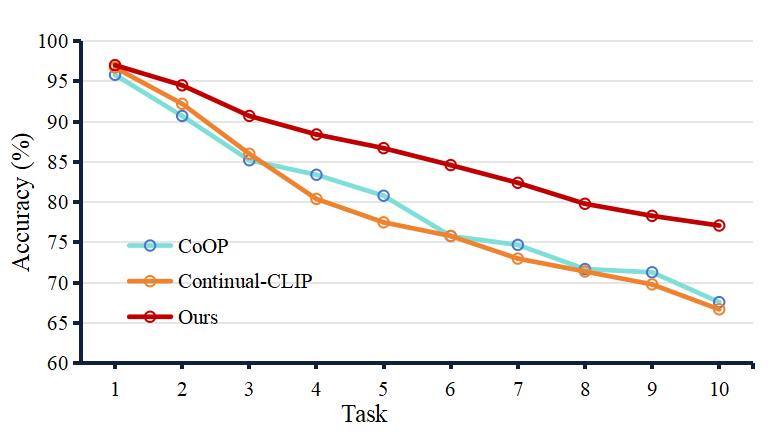}
  %\text{(a)}
\end{minipage}
\begin{minipage}[b]{0.3\textwidth}
  \centering
  \includegraphics[width=\textwidth]{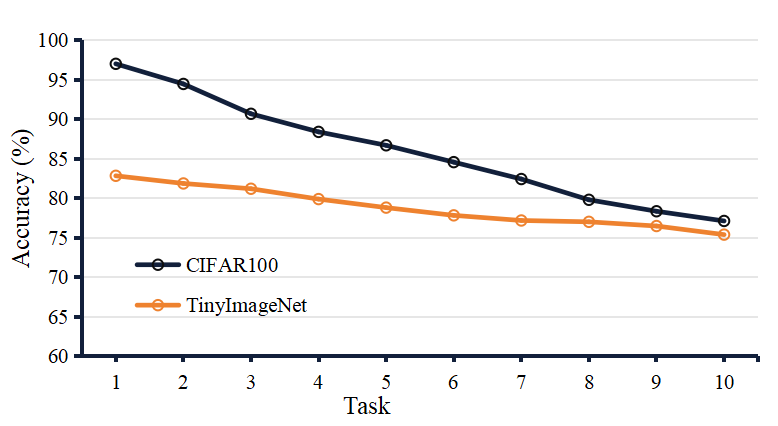}
  %\text{(a)}
\end{minipage}
\caption{Last accuracy of every Task in 10 steps. The Last accuracy of Task $t$, $t \in \{1, 2, . . . , 10 \} $ is the Top-1 accuracy over all the previous Tasks (i.e., Tasks $1, 2, . . . ,t$). The results in \textbf{left} and \textbf{middle} are conducted on CIFAR100. The result of "Ours" in \textbf{left} is based on ViT-L/14 and in \textbf{middle} and \textbf{right} are based on ViT-B/16.}
\label{fig:tentask}
\end{figure*}

\begin{figure*}[htbp]
\centering
    \begin{minipage}{0.4\textwidth}
        \includegraphics[width=\linewidth]{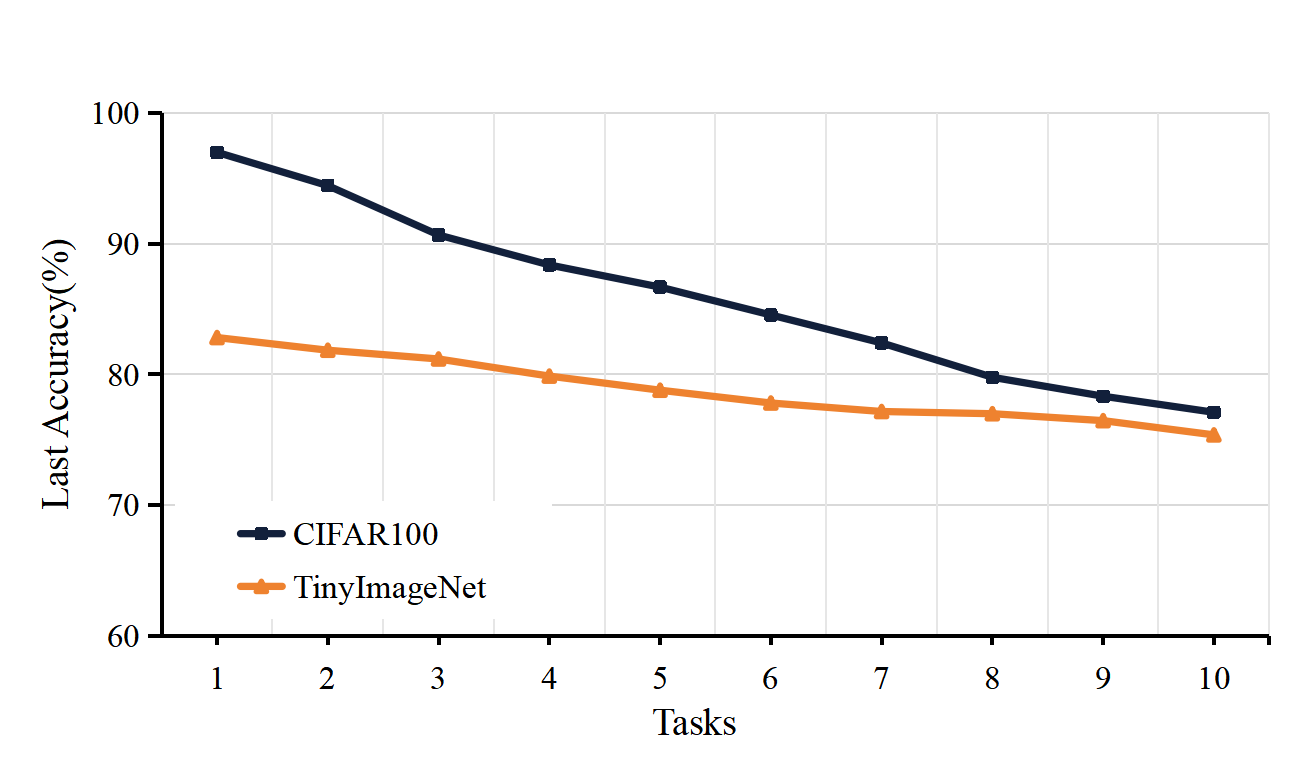}
    \end{minipage}
    \begin{minipage}{0.4\textwidth}
        \includegraphics[width=\linewidth]{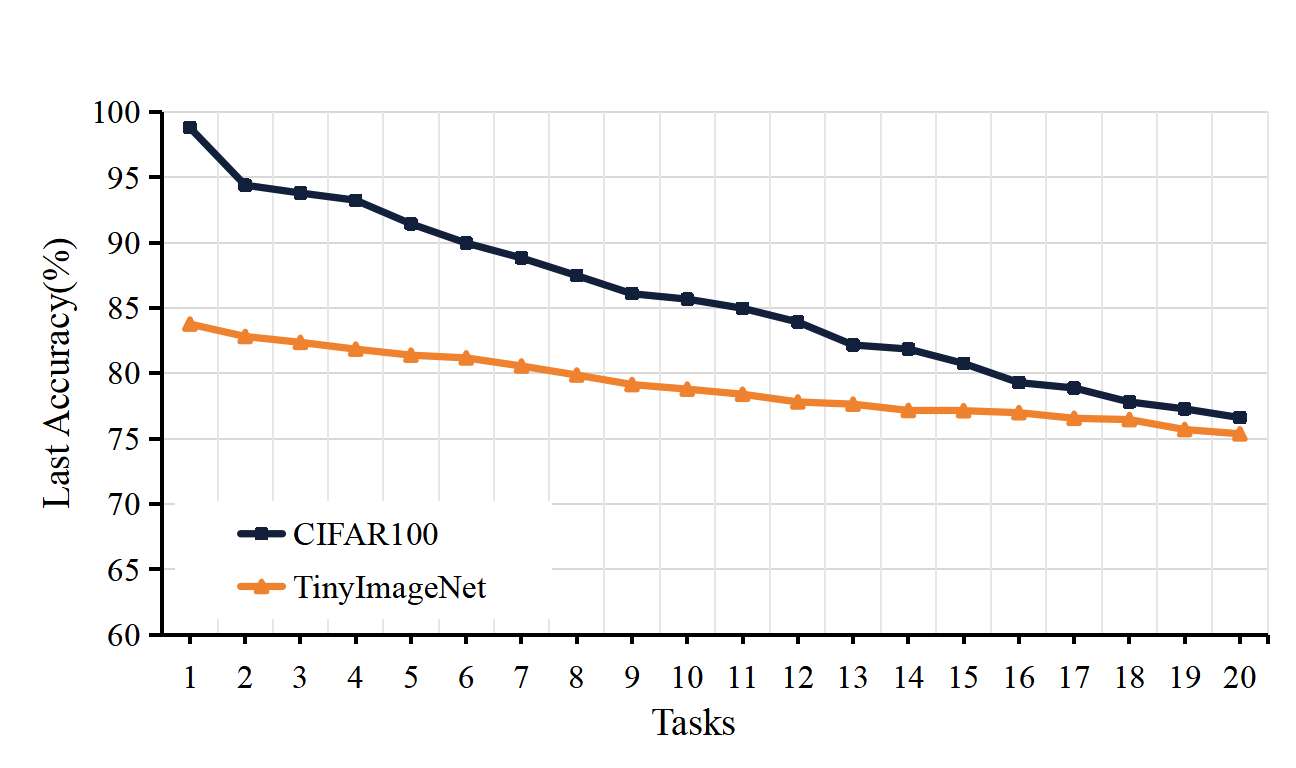}
    \end{minipage}
    \vspace{-0.5em}
    \begin{minipage}{0.4\textwidth}
        \includegraphics[width=\linewidth]{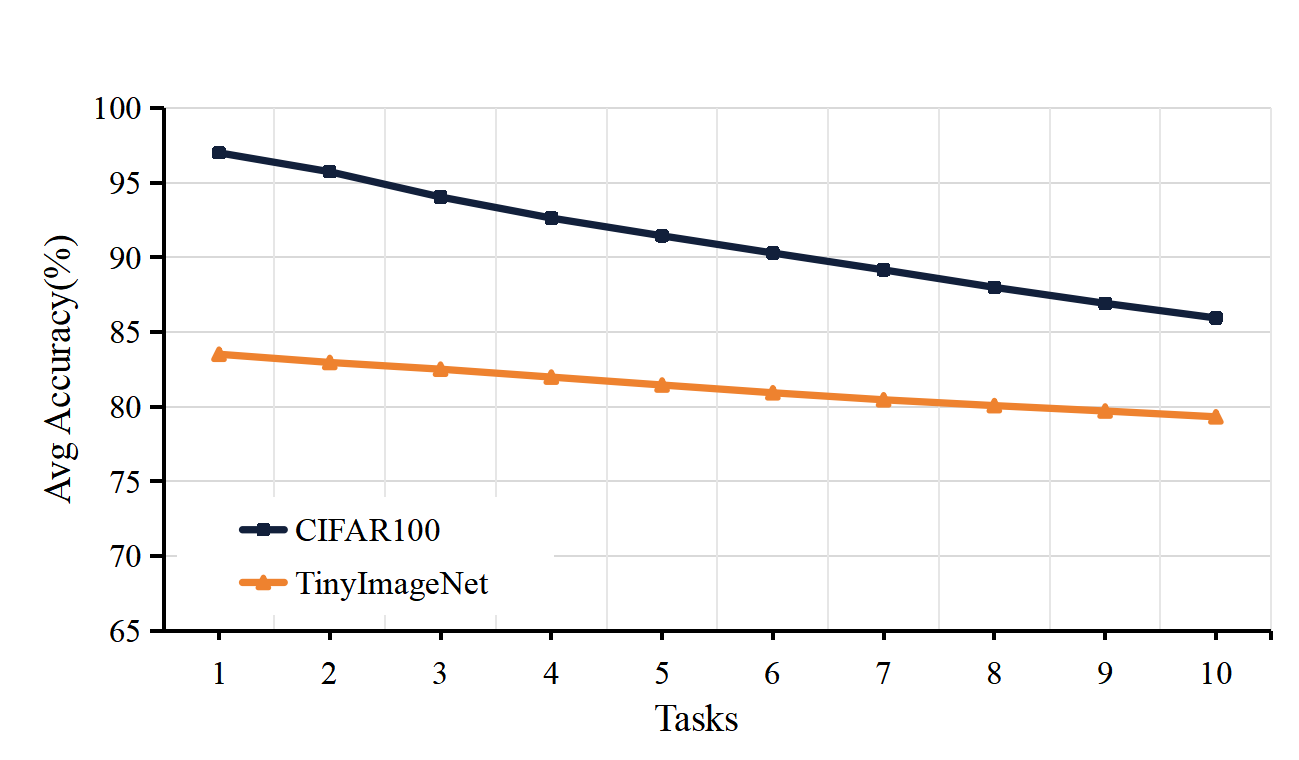}
    \end{minipage}
    \begin{minipage}{0.4\textwidth}
        \includegraphics[width=\linewidth]{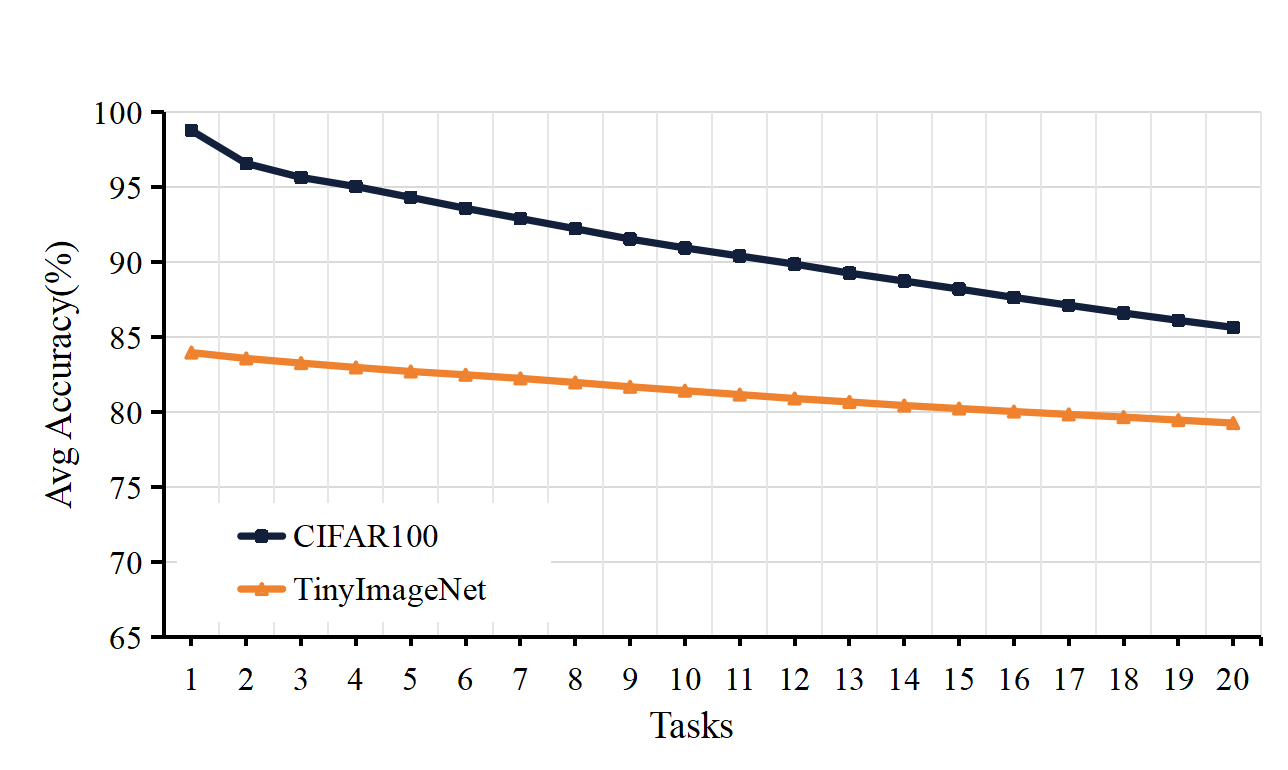}
    \end{minipage}
\caption{Avg accuracy and Last accuracy of every Task on CIFAR100 and TinyImageNet with CLIP ViT-B/16. The Last accuracy of Task $t$, $t \in \{1, 2, . . . , 20 \} $ is the Top-1 accuracy over all the previous Tasks (i.e., Task $1, 2, . . . ,t$). and the Avg accuracy is the average value of Last accuracy over all the previous Tasks. The 100 classes of TinyImageNet are used as base classes.}
\label{fig:Appendixcompacccafir100tinyimagenet}
\end{figure*}

\begin{table*}[htb]\scriptsize
\caption{Average accuracy of different continual learning methods on CIFAR100 with CLIP ViT-B/16. For example, "1-3" signifies that adapters are inserted in the first 3 blocks. The accuracy of Task $t$, $t \in \{1, 2, . . . , 10 \} $ reported here is the Top-1 accuracy over all the previous tasks (i.e., Tasks $1, 2, . . . ,t$). 'CLIP' means no adapter in model.}
\label{table:appendixadapter-bet-block-10steps}
\centering
\resizebox{1.0\textwidth}{!}{
\begin{tabular}{l|cccccccccc}
\toprule  
 Method     &  Task 1  &  Task 2  &  Task 3  & Task 4  & Task 5  & Task 6  & Task 7  & Task 8  &  Task 9  & Task 10  \\
\hline 
CLIP  & 91.60 & 89.60 & 85.30 & 82.60 & 80.48 & 78.13 & 75.17 & 72.65 & 71.30 & 70.08 \\
1-3   & 95.60 & 93.55 & 89.77 & 87.50 & 85.70 & 83.48 & 81.49 & 79.35 & 77.77 & 76.47 \\
1-6   & 95.80 & 93.20 & 89.90 & 87.45 & 85.84 & 83.87 & 81.94 & 79.66 & 78.04 & 76.65 \\
1-9   & 97.00 & 93.95 & 89.87 & 87.40 & 85.64 & 83.30 & 81.17 & 78.91 & 77.28 & 76.00 \\
1-12  & 97.10 & 94.35 & 90.37 & 87.90 & 86.10 & 84.02 & 82.00 & 79.49 & 78.00 & 76.70 \\
4-6   & 95.20 & 92.90 & 89.43 & 87.00 & 85.50 & 83.48 & 81.61 & 79.34 & 77.63 & 76.26 \\
4-9   & 96.60 & 93.80 & 89.73 & 87.15 & 85.30 & 82.90 & 80.89 & 78.62 & 77.16 & 75.60 \\
4-12  & 96.80 & 94.05 & 90.13 & 87.48 & 85.66 & 83.43 & 81.37 & 78.97 & 77.39 & 75.93 \\
7-9   & 95.40 & 92.90 & 89.13 & 86.28 & 84.22 & 81.82 & 79.57 & 76.81 & 75.17 & 73.72 \\
7-12  & 95.40 & 93.25 & 89.60 & 86.55 & 84.66 & 82.25 & 80.06 & 77.26 & 75.81 & 74.31 \\
10-12 & 94.20 & 92.00 & 88.53 & 85.65 & 83.62 & 81.15 & 78.74 & 76.11 & 74.53 & 73.18 \\
\toprule  
\end{tabular}}
\end{table*}

\begin{table}[htb]\scriptsize
\caption{Effect of number and position of adapters loaded in CLIP image encoder. For example, "1-3" signifies that adapters are inserted in the first 3 blocks. All experiments are conducted on CIFAR100 with CLIP ViT-B/16, 'CLIP' means no adapter in model. The best results are coloured gray}
\label{table:appendixadapter-bet-block-all}
\centering
\resizebox{0.48\textwidth}{!}{
\begin{tabular}{l|cccccc}
    \toprule
    &\multicolumn{2}{c}{10 steps} &\multicolumn{2}{c}{20 steps} &\multicolumn{2}{c}{50 steps} \\ 
    \bf Blocks & \bf Avg  & \bf Last & \bf Avg & \bf Last & \bf Avg & \bf Last \\ 
    \midrule
    CLIP & 79.69 & 70.08 & 80.41 & 70.08 & 80.80 & 70.08 \\
    1-3    & 85.07 & 76.47 & 84.76 & 75.66 & 83.10 & 72.86  \\
    1-6    & 85.23 & 76.65 & \cellcolor{lightgray}85.40 & 75.84 & 83.86 & 73.50  \\
    1-9    & 85.05 & 76.00 & 85.26 & 75.70 & 84.01 & 73.66  \\
    1-12   & \cellcolor{lightgray}85.60 & \cellcolor{lightgray}76.70 & 85.30 & \cellcolor{lightgray}76.02 & \cellcolor{lightgray}84.09 & \cellcolor{lightgray}73.77  \\
    4-6    & 84.84 & 76.26 & 85.08 & 75.58 & 82.30 & 71.59  \\
    4-9    & 84.78 & 75.60 & 84.60 & 74.65 & 82.94 & 72.31  \\
    4-12   & 85.12 & 75.93 & 84.81 & 75.03 & 83.09 & 72.46  \\
    7-9    & 83.50 & 73.72 & 83.40 & 73.26 & 81.79 & 71.14  \\
    7-12   & 83.91 & 74.31 & 83.50 & 73.54 & 81.96 & 71.42  \\
    10-12  & 82.77 & 73.18 & 82.86 & 72.89 & 81.66 & 71.16  \\
    \bottomrule
\end{tabular}}
\end{table}

\begin{table}[htb]
\caption{The influence of bottleneck dimension of adapters. All experiments are conducted on CIFAR100 with CLIP ViT-B/16. "Bottleneck" means the  projection dimension of adapters. Other experiment settings are the same as Appendix \ref{appendixsec:exp-expsetting}.}
\label{appendix:tabe-bottleneck}
\centering
\scriptsize
\resizebox{0.48\textwidth}{!}{
\begin{tabular}{c|cccc}
    \toprule
     &\multicolumn{2}{c}{10 steps} &\multicolumn{2}{c}{20 steps} \\
    Bottleneck & \bf Avg  & \bf Last & \bf Avg & \bf Last \\
    \midrule
     1   & 85.94	& 77.61	& 85.33	& 75.98\\
     2   & 86.14	& 77.68	& 85.50	& 76.16\\
     4   & 85.84	& 77.01 & 85.56 & 76.44 \\
     8   & 85.85	& 77.10 & 85.63 & 76.58 \\
     16  & 85.92	& 77.14 & 85.64 & 76.6  \\
     32  & 85.88	& 77.06 & 85.61 & 76.47 \\
     64  & 85.93	& 77.09 & 85.67 & 76.61 \\
     128 & 85.89	& 77.03 & 85.53 & 76.32 \\
     256 & 85.91 & 76.93  & 85.44 & 76.21\\
    \bottomrule
\end{tabular}}
\end{table}

\begin{figure*}[htb]
\centering
    \begin{minipage}{0.3\textwidth}
        \includegraphics[width=\linewidth]{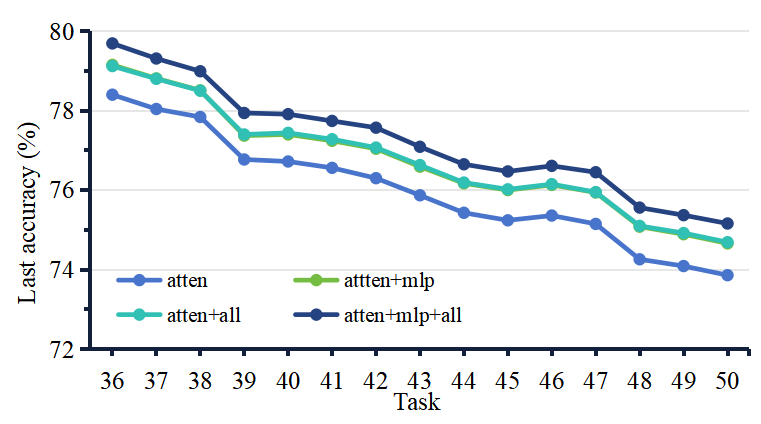}
        %\centering (a)
    \end{minipage}
    \begin{minipage}{0.3\textwidth}
        \includegraphics[width=\linewidth]{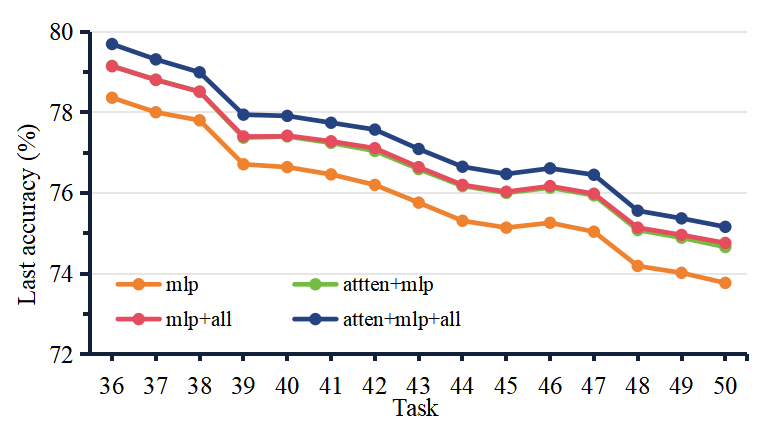}
        %\centering (b)
    \end{minipage}
    \begin{minipage}{0.3\textwidth}
        \includegraphics[width=\linewidth]{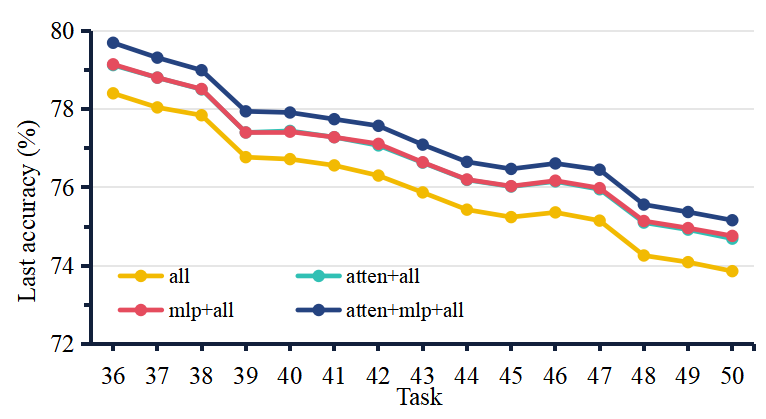}
        %\centering (c)
    \end{minipage}
\caption{The Last accuracy of different Multi-Adapter implementations on classes 70-100 in 50 steps. Every step contains 2 classes, so numbers like "36" means classes 70-72. Experiments are conducted on CIFAR100 with CLIP ViT-B/16 and the classes are tested with the same sequence.}
\label{fig:appendix-70-100acc}
\end{figure*}

\begin{table*}[htb]\scriptsize
\caption{The results of different implementations of Multi-Adapter with the bottleneck dimension being \textbf{1}. The structures of Adapter-MLP, Adapter-Atten and Adapt-All are shown in Fig.\ref{fig:ouradapter}. "Para" refers to trainable parameters, with "M" standing for million. All experiments are conducted on CIFAR100 and the best results are coloured grey}
\label{table:appendix-Multi-imple}
\centering
\tiny %\tiny \scriptsize \footnotesize \small \normalsize \large \Large
\renewcommand{\arraystretch}{0.7}
\resizebox{1.0\textwidth}{!}{
\begin{tabular}{ccc|c|cccccc}
    \toprule
    & & & & \multicolumn{2}{c}{10 steps} &\multicolumn{2}{c}{20 steps} &\multicolumn{2}{c}{50 steps} \\ 
    \bf Adapt-MLP  & \bf Adapt-Atten  & \bf Adapt-All & Para(M) & \bf Avg  & \bf Last & \bf Avg & \bf Last & \bf Avg & \bf Last \\ 
    \midrule
    \textcolor{lightgray}{\XSolidBrush} & \textcolor{lightgray}{\XSolidBrush} & \textcolor{lightgray}{\XSolidBrush} & 0 & 79.69 & 70.08 & 80.41 & 70.08 & 80.80 & 70.08 \\
    \CheckmarkBold & \textcolor{lightgray}{\XSolidBrush}   & \textcolor{lightgray}{\XSolidBrush} & 1.19 & 85.31 & 	76.72 &	85.47	&76.20	&83.93	&73.59\\
    \textcolor{lightgray}{\XSolidBrush}   & \CheckmarkBold & \textcolor{lightgray}{\XSolidBrush} & 1.19 & 85.94 & \cellcolor{lightgray}77.61 &	85.33	&75.98	&83.80	&73.50\\
    \textcolor{lightgray}{\XSolidBrush}   & \textcolor{lightgray}{\XSolidBrush}   & \CheckmarkBold & 1.19 & 85.48 &	76.19 &	85.48	&76.19	&83.93	&73.60\\
    \CheckmarkBold & \CheckmarkBold & \textcolor{lightgray}{\XSolidBrush} & 2.38 &\cellcolor{lightgray}85.98 &	77.41 &\cellcolor{lightgray}85.67\cellcolor{lightgray}&76.42	&84.60	&74.54\\
    \textcolor{lightgray}{\XSolidBrush}   & \CheckmarkBold & \CheckmarkBold & 2.38 & 85.90 &	77.36 &	85.53	&76.29	&84.60	&74.53\\
    \CheckmarkBold & \textcolor{lightgray}{\XSolidBrush}   & \CheckmarkBold & 2.38 & 85.76 &	77.20 & 85.31	&75.90	&84.51	&74.50\\
    \CheckmarkBold & \CheckmarkBold & \CheckmarkBold & 3.57 & 85.92 & 77.25 & 85.02	&75.48 &\cellcolor{lightgray}84.72 &\cellcolor{lightgray}74.80\\
    \bottomrule
\end{tabular}}
\end{table*}

\section{Additional implementation of SimE}
\label{appendixsec:meth-SimE-toymodel}

\textbf{The Prototype extractor.}
In image classification, we follow \cite{snell2017prototypical}, setting the average features of the classes as the weights of the classifier. For the $t$-th task ($t=2,\ldots, T$), we do not update the weights of $E(\bm{x})$ in (\ref{adapter-overallV1}); instead, we use the $E(\bm{x})$ directly to calculate the average value of features and set it as prototypes in datasets $\{ \bm x_i^1, \ldots, \bm x_i^t \}_{i=1}^{n_t}$:
\begin{equation}
    p_k = \frac{1}{K} \sum_{j=1}^{|D^t|} I(y_j = k)E(\bm x),
\end{equation}
where $p_k \in \mathbb{R}^d$ is the prototype of increment class $k$ in $t$-th task, $K = \sum_{j=1}^{\lvert |D^t| \rvert} I( y_j = k)$, $I(\cdot)$ is the indicator function. $p_k$ contains the average features of class $k$, implying that the images of class $k$ should exhibit the greatest similarity with $p_k$ among all prototypes. 

\textbf{The Classifier}.
In CIL, the classifier is dynamic and can be implemented in various ways\cite{wang2023attriclip, mai2021supervised}. In this paper, we use a FC layer as our classifier\cite{snell2017prototypical}. For the $t$-th task, the classifier is an FC layer $W_t \in \mathbb{R}^{D \times (N+M)}$, where $D$ is the feature dimension, $N$ is the number of classes at ($t$-1)-th task, $M$ is the number of increment classes in $t$-th task. We use the training set data $\bm x_{\text{pro}} = \{ (\bm x_i^t, \bm y_i^t) \}_{i=1}^{n_t}$ from the $t$-th task to calculate the prototype $W_{\text{pro}}=\text{mean}(E(\bm x_{\text{pro}}))$, where $n_t$ is the number of samples in the training set of the $t$-th task, then update the FC layer $W_t$: $W_{t}=W_{t-1}+W_{\text{pro}}$, where $W_{t-1} \in \mathbb{R}^{D \times N}$ and $W_{\text{pro}} \in \mathbb{R}^{D \times M}$. The cosine similarity for classification is then calculated as: 
\begin{equation}
\label{eq:simil}
f(\bm{x}) = \left( \frac{\mathbf{W}}{\|\mathbf{W}\|_2} \right)^\top \left( \frac{E(\bm{x})}{\|E(\bm{x})\|_2} \right)
\end{equation}
given that prototype $p_i$ is most similar to instances of class $i$, it is expected that the classifier will assign a higher probability to the correct class label.

\textbf{The toy model}. In SimE, there are numerous implementations for the encoder and adapter. Here, we first employ the CLIP visual encoder as the encoder and AdaptFormer or Multi-Adapter as the adapter to establish a toy model of SimE. Subsequently, we explore various specific implementations of SimE.
\textbf{CLIP} \cite{radford2021learning} is a powerful visual-language contrastive learning model comprising an image encoder and a text encoder. It is trained on 400 million image-text pairs and possesses strong feature extraction capabilities. In this paper, we employ the CLIP pre-trained visual encoder as our backbone, i.e., we instantiate our encoder using the CLIP pre-trained weights $\theta$ and $\phi$. Additionally, there exist CLIP models pre-trained on different datasets, corresponding to different encoder instances. It is noteworthy that during the continual learning process, the pre-trained weights of CLIP are frozen and do not participate in weight updates. 
\textbf{Adaptformer}\cite{chen2022adaptformer} containing a down-projection $W_{\text{down}} \in \mathbb{R}^{D \times R}$ to reduce the feature dimension, a non-linear activation function(ReLU) and an up-projection $W_{\text{up}} \in \mathbb{R}^{R \times D}$ to project the features back to their original dimension, where $D$ is the feature dimension and $R$ is the dimension of bottleneck. The specific form of the AdaptFormer is:
\begin{equation}
d_i(\Tilde{\eta_i},\bm{h}_i)=\alpha\mathrm{ReLU}(\bm{h}_i \cdot W_{\text{down}}) \cdot W_{\text{up}}
\label{eq.11}
\end{equation}
Here $\alpha$ is the scaling factor in the residual connection, which is set to 0.1 by default in AdapterFormer.
\textbf{Multi-Adapter}
is represented by (\ref{adapter-overallV2}). The Adaptformer is a special case in Multi-Adapter and we employ the same down-up projects in (\ref{eq.11}) to initialize the Multi-Adapter.

\section{Additional optimisation of adapter for domain adaptation}
\label{appendixsec:optimisation-of-Adapter}
The CLIP possesses exceptional zero-shot capabilities, enabling it to achieve performance comparable to supervised models on new datasets without finetuning. This underscores CLIP's powerful feature extraction abilities. However, CLIP still faces a domain gap between pre-trained datasets and downstream task datasets. For instance, while CLIP excels on datasets like ImageNet, its performance on MNIST\cite{radford2021learning} is poor. To bridge this gap, it is necessary to finetune CLIP for incremental learning downstream tasks. During the finetuning process, the weights $\theta_i, \phi_i$ of pre-trained CLIP image encoder $E'(\bm{c})$ in (\ref{adapter-overallV1}) are frozen, and only the adapters and classifier are updated:
\begin{equation}
E^{*}(\bm{c}) = F(E'(\bm{c}),D),
\end{equation}
where $E^{*}(x)$ is the adapted CLIP image encoder, $F$ denotes the finetuning process, $ D $ represents the data of the incremental tasks, $\eta$ refers to the trainable parameters, $\eta = \cup_{j=1}^{Z} \tilde{\eta}_{ij} \cup \theta_{W_t}$
Through finetuning, the pre-trained encoder can better adapt to downstream datasets. However, continuously finetuning on a series of tasks $D = \{D_1, \ldots, D_T\}$ would diminish its feature extraction capabilities due to Catastrophic Forgetting. Therefore, in this paper, we finetune the pre-trained encoder only on the first task $D_1$ to maximise the retention of the encoder's previous knowledge:
\begin{equation}
E^{*}(\bm{c}) = F(E'(\bm c_1),D_1) 
\label{eq:finetune1}.
\end{equation}
By finetuning the first task, the encoder can better adapt to downstream datasets. However, the finetuning process will inevitably diminish the zero-shot capabilities of the pre-trained encoder. To better preserve the feature extraction capabilities of the pre-trained encoder, we concatenate the features output by the pre-trained encoder and the finetuned encoder and feed this combined feature set into the classifier for image classification.
\begin{equation}
E^c(\bm{c}) = \{ E^{*}(\bm{c}); E(\bm{c}) \} 
\label{eq:concatenation}
\end{equation}
Here, \( E^c(\bm{c}) \) represents the composite encoder, and \( \{ \cdot; \cdot \} \) denotes the concatenation of the features output by the pre-trained encoder and the finetuned encoder. It is important to note that the concatenation is performed after finetuning, and encoder \( {E}^c(\bm{c}) \) will be frozen in subsequent tasks to maximise the retention of its feature extraction capabilities.

\section{Additional theorem proof}
\label{appendixsec:theorem proof}
\textbf{Proof of theorem \ref{thm:1}}

\noindent $Proof$.\quad For any $\theta_0\in \Theta_0$, there exists a corresponding $\theta_N\in \Theta_{N,\mathrm{loc}}$ 
such that the newly introduced Adapter parameters are set to the ``identity'' or ``off''. Hence
\[
  f\bigl(x;\,\theta_0,\,0,\,\varnothing\bigr)
  \;=\;
  f\bigl(x;\,\theta_N,\,N,\,\mathrm{loc}\bigr)
\]
which implies
\[
  \sup_{\theta_0 \in \Theta_0} f(\theta_0)
  \;\;\le\;\;
  \sup_{\theta_N \in \Theta_{N,\mathrm{loc}}} f(\theta_N)
\]
Similarly, $\Theta_{N,\mathrm{loc}}\subseteq\Theta_{M,\mathrm{loc}}$ yields
\[
  \sup_{\theta_N \in \Theta_{N,\mathrm{loc}}} f(\theta_N)
  \;\;\le\;\;
  \sup_{\theta_M \in \Theta_{M,\mathrm{loc}}} f(\theta_M)
\]
Combining these two inequalities finishes the proof:
\[
  \sup f(\Theta_{M,\mathrm{loc}}) \ge
  \sup f(\Theta_{N,\mathrm{loc}}) \ge \sup f(\Theta_0)
\]

% \textbf{Proof of theorem \ref{thm:2}}

% \noindent $Proof$. \quad Formally, for performance measure function with any parameters 
% \[ f: \Theta_{num,loc} \to \mathbb{R} \]
% It will be affected by both $num$ and $loc$ of adapters. If we consider adapter connect combination methods as \( T(\Theta) \), there is no guarantee that \( T(\Theta) \) will exactly return \( \arg\max_{\Theta} f(\theta) \). 

% As a result, for $M>N$, we only have 
% \[
% %\sup f(\Theta_{M,loc}) \ge \sup f(\Theta_{N,loc}),
% \sup_{\theta\in \Theta_{M,\mathrm{loc}}} f(\theta)>\sup_{\theta\in \Theta_{N,\mathrm{loc}}} f(\theta)
% \]
% but it is possible that 
% \[
% f(T_1(\Theta_M)) < f(T_2(\Theta_N)).
% \]

% This precisely reflects the practical scenario described in Theorem \ref{thm:2}.

\section{Additional experimental setting }
\label{appendixsec:exp-expsetting}
\textbf{Datasets.}
The main experiments are implemented on CIFAR100 and TinyImageNet. CIFAR100\cite{krizhevsky2009learning} consist of 60K images with size of \(32 \times 32\) from 100 classes, which are split into 2 classes, 5 classes and 10 classes in each step. Each class consist of 500 training and 100 testing samples. TinyImageNet\cite{zheng2023preventing}, as a subset of ImageNet, consists of 100K images with size of \(64 \times 64\) from 200 classes and each class consist of 500 training and 50 testing images. In our experiment, the 100 classes of TinyImageNet are split as base classes which are used for finetuning, and the rest 100 classes are split into 5 classes, 10 classes and 20 classes in each step. 

To evaluate scalability and robustness in more realistic and larger-scale scenarios, we further conduct experiments on: CUB200 \cite{wah2011caltech}, containing 11,788 bird images across 200 classes, with resolutions typically exceeding 300 pixels on the long side; ImageNet-R \cite{hendrycks2021many}, which contains 30K images from 200 ImageNet classes rendered in artistic, cartoon, and abstract styles; ImageNet-100, a 100-class subset of ImageNet with images of typical resolution exceeding 500 pixels; VTAB \cite{zhai2019large}, a collection of 19 diverse vision datasets with images resized to 224×224; ImageNet-A \cite{hendrycks2021natural}, an adversarially filtered subset of ImageNet containing 7.5K challenging images with long sides typically above 500 pixels; ObjectNet \cite{barbu2019objectnet}, which has 50K real-world object images with significant viewpoint, background, and rotation variations (long sides usually over 500 pixels); and OmniBenchmark \cite{zhang2022benchmarking}, a broad-spectrum benchmark with diverse image sizes, some exceeding 300×500 pixels. In this paper, we sample 200 classes from ObjectNet and ImageNet-A, and 5 datasets from VTAB, with each containing 10 classes, to conduct CIL experiments. All datasets in Table.\ref{tab:sota} are split into 10 steps, except
ObjectNet, which is split into 20 steps.

\textbf{Network architectures in the SimE.}
We have developed SimE, leveraging CLIP and adapters (Adaptformer \cite{chen2022adaptformer} and Multi-Adapter) for class-incremental learning tasks. Within SimE, the CLIP image processor is utilized for data preprocessing. The image encoder, featuring various backbone sizes such as ViT-B/16, ViT-B/32, and ViT-L/14, is finetuned using adapters across different pre-trained datasets, including WIT-400B, Laion-400M, Laion-2B, Datacomp-1B, and CommonPool-1B. The classifier employs a fully connected (FC) layer, which uses class prototypes as weights.

\textbf{Evaluation metrics.}
Following the methodology of \cite{rebuffi2017icarl}, we assess SimE and compare it with other baseline methods using two metrics: Average Accuracy (\textbf{Avg}) and Last Accuracy (\textbf{Last}). \textbf{Avg} represents the mean of the Top-1 accuracy for every task, while \textbf{Last} denotes the Top-1 accuracy of final task. Mathematically, for the $t$-th task, Average Accuracy is calculated as follows: \textbf{Avg}$ = \frac{1}{t}\sum_{i=1}^{t}$\textbf{Last}$_i$.

\textbf{Others CIL methods.} 
We compare the SimE with existing CLIP-based methods (e.g., CoOp\cite{zhou2022learning}, ZSCL\cite{zheng2023preventing}, Continual-CLIP\cite{thengane2022clip}, AttriCLIP\cite{wang2023attriclip}, Boosting-CL\cite{yu2024boosting}, L2P\cite{wang2020learn}, DualPrompt\cite{wang2022dualprompt}, CODA-Prompt\cite{smith2023coda}, APER\cite{zhou2025revisiting}) and typical continual learning methods (e.g., LwF\cite{li2017learning}, iCaRL\cite{rebuffi2017icarl}, DER\cite{yan2021dynamically}, iTAML\cite{rajasegaran2020itaml}, ARI\cite{wang2022anti}, UCIR\cite{hou2019learning}, PASS\cite{zhu2021prototype}, and DyTox\cite{douillard2022dytox}).

\textbf{Training procedures.} 
In this study, our experiments utilize the image encoder from CLIP\cite{radford2021learning}, and all experiments use 12 adapters (one per transformer block) unless stated otherwise. We finetune the SimE over 20 epochs on the first task for every dataset. Subsequently, all model weights, except the classifier, remain unchanged. During finetuning, we employ Stochastic Gradient Descent (SGD) as the optimizer. The starting learning rate is set at 0.01, adhering to a cosine decay schedule. We apply a weight decay of 0.0005, a batch size of 64, and the adapter’s bottleneck dimension is set to 64.

\section{Additional comparison on accuracy}
\label{appendixsec:addition-acc-eff}
In this section, we present additional accuracy comparisons of different CIL methods and datasets. To ensure fairness in backbone comparisons, we evaluated all methods using the same CLIP ViT-B/16 backbone to maintain a consistent basis for comparison. As shown in Table.\ref{table:same_backbone}, our method consistently outperforms the baselines in most cases, even under identical backbone settings. This confirms that the observed performance gains are not solely attributable to the backbone capacity, but rather to the effectiveness of our proposed learning strategy.

Class imbalance is a significant challenge in real-world continual learning scenarios. To investigate its impact, we conducted additional experiments under imbalanced settings. The results are reported in Table.\ref{table:class_imbalance}, where $imb\_factor$ denotes the imbalance ratio. Specifically, the number of samples for class $i$ is computed as $max\_num * (imb\_factor^{(i / (num\_classes)})$
where $max\_num$ is the original number of samples per class in the balanced setting, and $num\_classes$ is the total number of classes. The results in Table.\ref{table:class_imbalance} show that class imbalance indeed leads to a substantial drop in performance. In particular, when the imbalance factor decreases from 1 to 0.01 (i.e., from a fully balanced distribution to a highly imbalanced one), the model’s accuracy drops by up to 7.06\%. This confirms that class imbalance poses a considerable challenge for incremental learning models.

Besides, Table.\ref{table:Appendixcompacc10stepscafir100} and Fig.\ref{fig:tentask} show the Last accuracy for each task over 10 steps on CIFAR-100, compared with other CIL methods. We also compare the performance of our method on CIFAR-100 and TinyImageNet over 10 and 20 steps, as reported in Fig.\ref{fig:Appendixcompacccafir100tinyimagenet}.

\section{Additional influence of adapter components in SimE via ablation studies}
\label{appendixsec:add-adapter-abla}

In this section, we report additional experiments on the influence of adapter components in SimE. We first study the influence of adapter connections between transformer blocks and report it in Table.\ref{table:appendixadapter-bet-block-10steps} \& Table.\ref{table:appendixadapter-bet-block-all}, where "CLIP" indicates no adapter inserted in transformer blocks. All the experiments are conducted on CIFAR100 with CLIP ViT-B/16.

We also investigate the influence of adapter connections within transformer blocks, as illustrated in Fig.\ref{fig:appendix-70-100acc} \&Table.\ref{table:appendix-Multi-imple} \&Table.\ref{appendix:tabe-bottleneck}. Fig.\ref{fig:appendix-70-100acc} presents the model performance on classes 70-100 during advanced steps (50 steps) with various Multi-Adapter implementations. Table.\ref{table:appendix-Multi-imple} reports the results of different Multi-Adapter implementations across all steps, with the first row indicating the absence of an adapter in the encoder. Furthermore, we examine the influence of adapter's bottleneck dimension within the Multi-Adapter framework and report results in Table.\ref{appendix:tabe-bottleneck} with experiments conducted on Adapt-Atten.


\begin{thebibliography}{34}

\bibitem{goodfellow2013empirical}I. J. Goodfellow, M. Mirza, D. Xiao, A. Courville, and Y. Bengio, ``An empirical investigation of catastrophic forgetting in gradient-based neural networks,'' \emph{arXiv preprint arXiv:1312.6211}, 2013.

\bibitem{de2021continual}M. De Lange, R. Aljundi, M. Masana, S. Parisot, X. Jia, A. Leonardis, G. Slabaugh, and T. Tuytelaars, ``A continual learning survey: Defying forgetting in classification tasks,'' \emph{IEEE Trans. Pattern Anal. Mach. Intell.}, vol. 44, no. 7, pp. 3366--3385, Jul. 2021.

\bibitem{masana2022class}M. Masana, X. Liu, B. Twardowski, M. Menta, A. D. Bagdanov, and J. Van De Weijer, ``Class-incremental learning: survey and performance evaluation on image classification,'' \emph{IEEE Trans. Pattern Anal. Mach. Intell.}, vol. 45, no. 5, pp. 5513--5533, May 2022.

\bibitem{li2017learning}Z. Li and D. Hoiem, ``Learning without forgetting,'' \emph{IEEE Trans. Pattern Anal. Mach. Intell.}, vol. 40, no. 12, pp. 2935--2947, Dec. 2017.

\bibitem{serra2018overcoming}J. Serra, D. Suris, M. Miron, and A. Karatzoglou, ``Overcoming catastrophic forgetting with hard attention to the task,'' in \emph{Proc. Int. Conf. Mach. Learn.}, 2018, pp. 4548--4557.

\bibitem{rebuffi2017icarl}S.-A. Rebuffi, A. Kolesnikov, G. Sperl, and C. H. Lampert, ``iCaRL: Incremental classifier and representation learning,'' in \emph{Proc. IEEE Conf. Comput. Vis. Pattern Recog.}, 2017, pp. 2001--2010.

\bibitem{radford2021learning}A. Radford, J. W. Kim, C. Hallacy, A. Ramesh, G. Goh, S. Agarwal, G. Sastry, A. Askell, P. Mishkin, J. Clark, and others, ``Learning transferable visual models from natural language supervision,'' in \emph{Proc. Int. Conf. Mach. Learn.}, 2021, pp. 8748--8763.

\bibitem{thengane2022clip}V. Thengane, S. Khan, M. Hayat, and F. Khan, ``CLIP model is an efficient continual learner,'' \emph{arXiv preprint arXiv:2210.03114}, 2022.

\bibitem{ding2022don}Y. Ding, L. Liu, C. Tian, J. Yang, and H. Ding, ``Don't stop learning: Towards continual learning for the CLIP model,'' \emph{arXiv preprint arXiv:2207.09248}, 2022.

\bibitem{zheng2023preventing}Z. Zheng, M. Ma, K. Wang, Z. Qin, X. Yue, and Y. You, ``Preventing zero-shot transfer degradation in continual learning of vision-language models,'' in \emph{Proc. IEEE/CVF Int. Conf. Comput. Vis.}, 2023, pp. 19125--19136.

\bibitem{zhou2022learning}K. Zhou, J. Yang, C. C. Loy, and Z. Liu, ``Learning to prompt for vision-language models,'' \emph{Int. J. Comput. Vis.}, vol. 130, no. 9, pp. 2337--2348, Sep. 2022.

\bibitem{wang2023attriclip}R. Wang, X. Duan, G. Kang, J. Liu, S. Lin, S. Xu, J. L{\"u}, and B. Zhang, ``AttriCLIP: A non-incremental learner for incremental knowledge learning,'' in \emph{Proc. IEEE/CVF Conf. Comput. Vis. Pattern Recog.}, 2023, pp. 3654--3663.

\bibitem{yu2024boosting}J. Yu, Y. Zhuge, L. Zhang, D. Wang, H. Lu, and Y. He, ``Boosting Continual Learning of Vision-Language Models via Mixture-of-Experts Adapters,'' \emph{arXiv preprint arXiv:2403.11549}, 2024.

\bibitem{houlsby2019parameter}N. Houlsby, A. Giurgiu, S. Jastrzebski, B. Morrone, Q. De Laroussilhe, A. Gesmundo, M. Attariyan, and S. Gelly, ``Parameter-efficient transfer learning for NLP,'' in \emph{Proc. Int. Conf. Mach. Learn.}, 2019, pp. 2790--2799.

\bibitem{chen2022adaptformer}S. Chen, C. Ge, Z. Tong, J. Wang, Y. Song, J. Wang, and P. Luo, ``Adaptformer: Adapting vision transformers for scalable visual recognition,'' \emph{Adv. Neural Inform. Process. Syst.}, vol. 35, pp. 16664--16678, 2022.

\bibitem{gadre2024datacomp}S. Y. Gadre, G. Ilharco, A. Fang, J. Hayase, G. Smyrnis, T. Nguyen, R. Marten, M. Wortsman, D. Ghosh, J. Zhang, and others, ``Datacomp: In search of the next generation of multimodal datasets,'' \emph{Adv. Neural Inform. Process. Syst.}, vol. 36, 2024.

\bibitem{cherti2023reproducible}M. Cherti, R. Beaumont, R. Wightman, M. Wortsman, G. Ilharco, C. Gordon, C. Schuhmann, L. Schmidt, and J. Jitsev, ``Reproducible scaling laws for contrastive language-image learning,'' in \emph{Proc. IEEE/CVF Conf. Comput. Vis. Pattern Recog.}, 2023, pp. 2818--2829.

\bibitem{aljundi2018memory}R. Aljundi, F. Babiloni, M. Elhoseiny, M. Rohrbach, and T. Tuytelaars, ``Memory aware synapses: Learning what (not) to forget,'' in \emph{Proc. Eur. Conf. Comput. Vis.}, 2018, pp. 139--154.

\bibitem{kirkpatrick2017overcoming}J. Kirkpatrick, R. Pascanu, N. Rabinowitz, J. Veness, G. Desjardins, A. A. Rusu, K. Milan, J. Quan, T. Ramalho, A. Grabska-Barwinska, and others, ``Overcoming catastrophic forgetting in neural networks,'' \emph{Proc. Natl. Acad. Sci.}, vol. 114, no. 13, pp. 3521--3526, Mar. 2017.

\bibitem{mallya2018packnet}A. Mallya and S. Lazebnik, ``Packnet: Adding multiple tasks to a single network by iterative pruning,'' in \emph{Proc. IEEE Conf. Comput. Vis. Pattern Recog.}, 2018, pp. 7765--7773.

\bibitem{wang2020learn}Z. Wang, T. Jian, K. Chowdhury, Y. Wang, J. Dy, and S. Ioannidis, ``Learn-prune-share for lifelong learning,'' in \emph{Proc. IEEE Int. Conf. Data Min.}, 2020, pp. 641--650.

\bibitem{buzzega2020dark}P. Buzzega, M. Boschini, A. Porrello, D. Abati, and S. Calderara, ``Dark experience for general continual learning: a strong, simple baseline,'' \emph{Adv. Neural Inform. Process. Syst.}, vol. 33, pp. 15920--15930, 2020.

\bibitem{cha2021co2l}H. Cha, J. Lee, and J. Shin, ``Co2l: Contrastive continual learning,'' in \emph{Proc. IEEE/CVF Int. Conf. Comput. Vis.}, 2021, pp. 9516--9525.

\bibitem{wang2022learning}Z. Wang, Z. Zhang, C.-Y. Lee, H. Zhang, R. Sun, X. Ren, G. Su, V. Perot, J. Dy, and T. Pfister, ``Learning to prompt for continual learning,'' in \emph{Proc. IEEE/CVF Conf. Comput. Vis. Pattern Recog.}, 2022, pp. 139--149.

\bibitem{wang2022dualprompt}Z. Wang, Z. Zhang, S. Ebrahimi, R. Sun, H. Zhang, C.-Y. Lee, X. Ren, G. Su, V. Perot, J. Dy, and others, ``Dualprompt: Complementary prompting for rehearsal-free continual learning,'' in \emph{Proc. Eur. Conf. Comput. Vis.}, 2022, pp. 631--648.

\bibitem{dong2024efficient}W. Dong, D. Yan, Z. Lin, and P. Wang, ``Efficient adaptation of large vision transformer via adapter re-composing,'' \emph{Adv. Neural Inform. Process. Syst.}, vol. 36, 2024.

\bibitem{gao2024clip}P. Gao, S. Geng, R. Zhang, T. Ma, R. Fang, Y. Zhang, H. Li, and Y. Qiao, ``Clip-adapter: Better vision-language models with feature adapters,'' \emph{Int. J. Comput. Vis.}, vol. 132, no. 2, pp. 581--595, 2024.

\bibitem{shi2024clip}K. Shi, J. Lu, Z. Fang, and G. Zhang, ``Clip-enhanced unsupervised domain adaptation with consistency regularization,'' in \emph{Proc. IEEE Int. Joint Conf. Neural Netw.}, 2024, pp. 1--8.

\bibitem{liu2023class}X. Liu, X. Cao, H. Lu, J.-W. Xiao, A. D. Bagdanov, and M.-M. Cheng, ``Class Incremental Learning with Pre-trained Vision-Language Models,'' \emph{arXiv preprint arXiv:2310.20348}, 2023.

\bibitem{ermis2022memory}B. Ermis, G. Zappella, M. Wistuba, A. Rawal, and C. Archambeau, ``Memory efficient continual learning with transformers,'' \emph{Adv. Neural Inform. Process. Syst.}, vol. 35, pp. 10629--10642, 2022.

\bibitem{ermis2022continual}B. Ermis, G. Zappella, M. Wistuba, A. Rawal, and C. Archambeau, ``Continual learning with transformers for image classification,'' in \emph{Proc. IEEE/CVF Conf. Comput. Vis. Pattern Recog.}, 2022, pp. 3774--3781.

\bibitem{hou2019learning}S. Hou, X. Pan, C. C. Loy, Z. Wang, and D. Lin, ``Learning a unified classifier incrementally via rebalancing,'' in \emph{Proc. IEEE/CVF Conf. Comput. Vis. Pattern Recog.}, 2019, pp. 831--839.

\bibitem{zhu2021prototype}F. Zhu, X.-Y. Zhang, C. Wang, F. Yin, and C.-L. Liu, ``Prototype augmentation and self-supervision for incremental learning,'' in \emph{Proc. IEEE/CVF Conf. Comput. Vis. Pattern Recog.}, 2021, pp. 5871--5880.

\bibitem{douillard2022dytox}A. Douillard, A. Ramé, G. Couairon, and M. Cord, ``Dytox: Transformers for continual learning with dynamic token expansion,'' in \emph{Proc. IEEE/CVF Conf. Comput. Vis. Pattern Recog.}, 2022, pp. 9285--9295.

\bibitem{yan2021dynamically}S. Yan, J. Xie, and X. He, ``DER: Dynamically expandable representation for class incremental learning,'' in \emph{Proc. IEEE/CVF Conf. Comput. Vis. Pattern Recog.}, 2021, pp. 3014--3023.

\bibitem{zhou2025revisiting}D.-W. Zhou, Z.-W. Cai, H.-J. Ye, D.-C. Zhan, and Z. Liu, ``Revisiting class-incremental learning with pre-trained models: Generalizability and adaptivity are all you need,'' \emph{Int. J. Comput. Vis.}, vol. 133, no. 3, pp. 1012--1032, 2025.

\bibitem{kang2025advancing}Z. Kang, L. Wang, X. Zhang, and K. Alahari, ``Advancing prompt-based methods for replay-independent general continual learning,'' \emph{arXiv preprint arXiv:2503.00677}, 2025.

\bibitem{schuhmann2022laion}C. Schuhmann, R. Beaumont, R. Vencu, C. Gordon, R. Wightman, M. Cherti, T. Coombes, A. Katta, C. Mullis, M. Wortsman, and others, ``Laion-5b: An open large-scale dataset for training next generation image-text models,'' \emph{Adv. Neural Inform. Process. Syst.}, vol. 35, pp. 25278--25294, 2022.

\bibitem{van2008visualizing}L. Van der Maaten and G. Hinton, ``Visualizing data using t-SNE,'' \emph{J. Mach. Learn. Res.}, vol. 9, no. 11, pp. 2579--2605, 2008.

\bibitem{mai2021supervised}Z. Mai, R. Li, H. Kim, and S. Sanner, ``Supervised contrastive replay: Revisiting the nearest class mean classifier in online class-incremental continual learning,'' in \emph{Proc. IEEE/CVF Conf. Comput. Vis. Pattern Recog.}, 2021, pp. 3589--3599.

\bibitem{snell2017prototypical}J. Snell, K. Swersky, and R. Zemel, ``Prototypical networks for few-shot learning,'' \emph{Adv. Neural Inform. Process. Syst.}, vol. 30, 2017.

\bibitem{krizhevsky2009learning}A. Krizhevsky, G. Hinton, and others, ``Learning multiple layers of features from tiny images,'' 2009.

\bibitem{wah2011caltech}C. Wah, S. Branson, P. Welinder, P. Perona, and S. Belongie, ``The caltech-ucsd birds-200-2011 dataset,'' California Institute of Technology, 2011.

\bibitem{hendrycks2021many}D. Hendrycks, S. Basart, N. Mu, S. Kadavath, F. Wang, E. Dorundo, R. Desai, T. Zhu, S. Parajuli, M. Guo, et al., ``The many faces of robustness: A critical analysis of out-of-distribution generalization,'' in \emph{Proc. IEEE Int. Conf. Comput. Vis.}, 2021, pp. 8340--8349.

\bibitem{zhai2019large}X. Zhai, J. Puigcerver, A. Kolesnikov, P. Ruyssen, C. Riquelme, M. Lucic, J. Djolonga, A. S. Pinto, M. Neumann, A. Dosovitskiy, et al., ``A large-scale study of representation learning with the visual task adaptation benchmark,'' \emph{arXiv preprint arXiv:1910.04867}, 2019.

\bibitem{hendrycks2021natural}D. Hendrycks, K. Zhao, S. Basart, J. Steinhardt, and D. Song, ``Natural adversarial examples,'' in \emph{Proc. IEEE Conf. Comput. Vis. Pattern Recognit.}, 2021, pp. 15262--15271.

\bibitem{barbu2019objectnet}A. Barbu, D. Mayo, J. Alverio, W. Luo, C. Wang, D. Gutfreund, J. Tenenbaum, and B. Katz, ``Objectnet: A large-scale bias-controlled dataset for pushing the limits of object recognition models,'' \emph{Adv. Neural Inf. Process. Syst.}, vol. 32, 2019.

\bibitem{zhang2022benchmarking}Y. Zhang, Z. Yin, J. Shao, and Z. Liu, ``Benchmarking omni-vision representation through the lens of visual realms,'' in \emph{Proc. Eur. Conf. Comput. Vis.}, 2022, pp. 594--611.

\bibitem{deng2009imagenet}J. Deng, W. Dong, R. Socher, L.-J. Li, K. Li, and L. Fei-Fei, ``Imagenet: A large-scale hierarchical image database,'' in \emph{Proc. IEEE Conf. Comput. Vis. Pattern Recognit.}, 2009, pp. 248--255.

\bibitem{smith2023coda}J. S. Smith, L. Karlinsky, V. Gutta, P. Cascante-Bonilla, D. Kim, et al., ``Coda-prompt: Continual decomposed attention-based prompting for rehearsal-free continual learning,'' in \emph{Proc. IEEE Conf. Comput. Vis. Pattern Recog.}, 2023, pp. 11909--11919.

\bibitem{rajasegaran2020itaml}J. Rajasegaran, S. Khan, M. Hayat, F. S. Khan, and M. Shah, ``iTAML: An incremental task-agnostic meta-learning approach,'' in \emph{Proc. IEEE/CVF Conf. Comput. Vis. Pattern Recog.}, 2020, pp. 13588--13597.

\bibitem{wang2022anti}R. Wang, Y. Bao, B. Zhang, J. Liu, W. Zhu, and G. Guo, ``Anti-retroactive interference for lifelong learning,'' in \emph{Proc. Eur. Conf. Comput. Vis.}, 2022, pp. 163--178.

\end{thebibliography}
\end{document}